\documentclass{article}

\usepackage{microtype}
\usepackage{graphicx}
\usepackage{subcaption}
\usepackage{booktabs} 
\usepackage[table]{xcolor}
\usepackage{hyperref}
\usepackage{tikz}
\tikzset{sepColor/.style={color=gray!70}}

\usepackage[accepted]{icml2026}

\usepackage{amsmath}
\usepackage{amssymb}
\usepackage{mathtools}
\usepackage{amsthm}
\usepackage{pgfplots}
\pgfplotsset{compat=1.18}
\usepackage{pgfplotstable}
\usepackage{multirow}
\pgfplotsset{compat=1.18}
\usepackage[capitalize,noabbrev]{cleveref}
\usepackage{algorithm}
\usepackage{xurl}
\theoremstyle{plain}
\newtheorem{theorem}{Theorem}[section]
\newtheorem*{theorem*}{Theorem}
\newtheorem{proposition}[theorem]{Proposition}

\theoremstyle{definition}

\theoremstyle{remark}

\usepackage[textsize=tiny]{todonotes}

\icmltitlerunning{TokenRatio: Principled Token-Level Preference Optimization via Ratio Matching}

\usepackage{tikz}
\usepackage{array}
\usepackage{booktabs}
\definecolor{WinGreen}{RGB}{120, 190, 120}     
\definecolor{TieYellow}{RGB}{250,200,99}   
\definecolor{LossRed}{RGB}{241,124,124}    
\usepackage{pgfplots}
\pgfplotsset{compat=1.18}
\usepackage{adjustbox}

\newcommand{\WinBar}[3]{%
\begin{tikzpicture}[baseline]
  \def\W{3.6}    
  \def\H{0.5}   
  \def\S{0.036}  

  \fill[WinGreen]   (0,0) rectangle (#1*\S,\H);
  \fill[TieYellow] (#1*\S,0) rectangle ({(#1+#2)*\S},\H);
  \fill[LossRed]   ({(#1+#2)*\S},0) rectangle (\W,\H);

  \ifdim #1pt>10pt
    \node[font=\scriptsize] at ({#1*\S/2},\H/2) {#1\%};
  \fi
  \ifdim #2pt>10pt
    \node[font=\scriptsize] at ({(#1+#2/2)*\S},\H/2) {#2\%};
  \fi
  \ifdim #3pt>10pt
    \node[font=\scriptsize] at ({\W-(#3*\S/2)},\H/2) {#3\%};
  \fi
\end{tikzpicture}%
}

\begin{document}

\twocolumn[
  \icmltitle{TokenRatio: Principled Token-Level Preference Optimization via Ratio Matching}



  \icmlsetsymbol{equal}{*}

  \begin{icmlauthorlist}
    \icmlauthor{Truong Nguyen}{equal,hust}
    \icmlauthor{Tien-Phat Nguyen}{equal,hust}
    \icmlauthor{Linh Ngo Van}{hust}
    \icmlauthor{Duy Minh Ho Nguyen}{unistuttgart,dfki,max}
    \icmlauthor{Khoa D. Doan}{vin}
    \icmlauthor{Trung Le}{monash}
  \end{icmlauthorlist}

  \icmlaffiliation{hust}{Hanoi University of Science and Technology}
   \icmlaffiliation{unistuttgart}{University of Stuttgart}
    \icmlaffiliation{dfki}{German Research Center for Artificial Intelligence}
  \icmlaffiliation{max}{Max Planck Research School for Intelligent Systems}
  \icmlaffiliation{vin}{VinUniversity}
  \icmlaffiliation{monash}{Monash University}

  \icmlcorrespondingauthor{Linh Ngo Van}{linhnv@soict.hust.edu.vn}
  \icmlcorrespondingauthor{Trung Le}{trunglm@monash.edu}

  \icmlkeywords{Machine Learning, ICML}

  \vskip 0.3in
]



\printAffiliationsAndNotice{\icmlEqualContribution}

\begin{abstract}
Direct Preference Optimization (DPO) is a widely used RL-free method for aligning language models from pairwise preferences, but it models preferences over full sequences even though generation is driven by per-token decisions. Existing token-level extensions typically decompose a sequence-level Bradley–Terry objective across timesteps, leaving per-prefix (state-wise) optimality implicit. We study how to recover \textbf{token-level} preference optimality using only standard sequence-level pairwise comparisons.
We introduce \textbf{Token-level Bregman Preference Optimization (TBPO)}, which posits a token-level Bradley--Terry preference model over next-token actions conditioned on the prefix, and derive a Bregman-divergence density-ratio matching objective that generalizes the logistic/DPO loss while preserving the optimal policy induced by the token-level model and maintaining DPO-like simplicity. We introduce two instantiations: TBPO-Q, which explicitly learns a lightweight state baseline, and TBPO-A, which removes the baseline through advantage normalization.
Across instruction following, helpfulness/harmlessness, and summarization benchmarks, TBPO improves alignment quality and training stability and increases output diversity relative to strong sequence-level and token-level baselines. Code is available at \url{https://github.com/dinhtruongng/TBPO}.
\end{abstract}

\section{Introduction}

Aligning large language models with human preferences is a central goal of modern fine-tuning, and recent work has increasingly favored lightweight and stable objectives for this purpose. In particular, Direct Preference Optimization (DPO) \cite{rafailov2023dpo} avoids the engineering complexity and variance of RL-style updates while still producing visibly improved behavior. However, a fundamental mismatch remains: language models \textbf{generate} text by making a sequence of local decisions, one token at a time conditioned on a prefix, while most preference optimization objectives treat each completion as a single indivisible object. Although this sequence-level view is convenient, it strains the connection between \textbf{where} a model is optimized (a whole response) and \textbf{how} it actually behaves at inference time (a distribution over next tokens at every state).

A natural response to the mismatch between autoregressive generation and sequence-level preference objectives is to push preference optimization to token granularity. However, existing “token-level” methods (e.g., TDPO \cite{zeng2024tokenleveldirectpreferenceoptimization} and TISDPO \cite{tis-dpo}) still fundamentally optimize a sequence-level Bradley–Terry preference model, differing mainly in how the sequence-level signal is distributed across timesteps. As a result, token-wise optimality is not explicitly enforced but remains implicit. This limitation becomes increasingly problematic for long generations: preference supervision is only observed at the level of full responses, making credit assignment to individual tokens indirect and allowing small early-token errors to cascade through subsequent decisions, even as the sequence-level objective improves. What we ultimately want, but rarely formalize, is a policy that is \textbf{preference-optimal at every state}, i.e., one that makes locally optimal tradeoffs at each state before the remainder of the completion unfolds. Achieving this stronger notion of optimality without sacrificing the simplicity and efficiency that make DPO attractive is therefore both challenging and valuable.


This paper proposes \textbf{Token-level Bregman Preference Optimization (TBPO)}, a preference-optimization framework that explicitly models preferences at the token level via density ratio matching while relying solely on sequence-level comparison data. TBPO is built on a token-level Bradley--Terry preference model that aligns preference learning with the autoregressive decision process of language models, and is designed to recover a token-wise optimal policy instead of implicitly relying on sequence-level credit assignment distributed across timesteps. Depending on the choice of scoring function in the token-level Bradley--Terry model, TBPO naturally yields different variants. We introduce two instantiations: \textbf{TBPO-Q}, which defines the token-level score via a state--action value function, and \textbf{TBPO-A}, which defines it via an advantage function. By casting TBPO as density-ratio matching under a Bregman-divergence objective, we even obtain a plug-and-play family of pairwise preference losses: small tweaks to the Bregman generator yield different objectives. Crucially, this flexibility does not come at the cost of theory or usability: the resulting objective still supports an optimality characterization in which the minimizer corresponds to the desired optimal token-level policy, while retaining DPO-style simplicity without fitting an explicit reward model or running on-policy RL.


We summarize our contributions as follows:
\begin{itemize}
    \item \textbf{Token-level Bregman Preference Optimization (TBPO).}
    We propose TBPO, a token-level preference-optimization framework that models preferences directly at the token level and learns from sequence-level comparisons.
    \item \textbf{Theory.}
    We provide a rigorous derivation and an optimality characterization: under the token-level preference model, minimizing the TBPO objective recovers a token-level optimal policy.
    \item \textbf{Empirical results.}
    Across benchmarks, model backbones, and tasks, TBPO improves alignment while also enhancing training stability and preserving response diversity relative to strong baselines.
\end{itemize}

\section{Related Work}

\textbf{Preference-based alignment.}
Aligning large language models (LLMs) with human judgments is commonly approached via preference-based optimization.
RLHF learns an explicit reward model from pairwise comparisons and optimizes the policy via reinforcement learning, but suffers from instability and high computational cost \citep{christiano2017deep,ouyang2022training}.
DPO provides an RL-free alternative by deriving a closed-form relationship between rewards and optimal policies under a KL-regularized objective, yielding a supervised loss equivalent to fitting a Bradley--Terry model over policy likelihood ratios \citep{rafailov2023dpo}.
Owing to its simplicity and stability, DPO has become a standard alignment baseline.

\textbf{DPO variants.}
Numerous extensions modify DPO through alternative feedback forms, regularization schemes, or reference models \citep{azar2023ipo,ethayarajh2024kto,xu2023cp,hong2024orpo,meng2024simpo, beta-dpo}.
Recent work also targets more specific failure modes: CDPO applies causal adjustment to reduce confounding in preference data, while logit-space SAM stabilizes DPO by applying sharpness-aware perturbations directly to policy logits~\citep{le2026causal,luo2026logitssam}.
While improving optimization dynamics or robustness, these methods largely preserve the \emph{sequence-level} preference modeling assumption.
Recent surveys summarize this rapidly expanding design space \citep{xiao2024survey,liu2024survey}.

\textbf{Token-level methods.}
To better regulate optimization dynamics in autoregressive models, several token-level variants of DPO have been proposed. TDPO decomposes the sequence-level DPO objective into token-wise terms with per-token KL constraints, while TIS-DPO further incorporates token-level importance sampling to improve training stability~\citep{zeng2024tokenleveldirectpreferenceoptimization,tis-dpo}. However, both methods still assume a \emph{sequence-level} Bradley–Terry preference model, redistributing sequence-level supervision across timesteps rather than explicitly defining token-level preferences. As a result, per-state optimality remains implicit rather than being directly enforced.
Beyond direct preference optimization, SWIFT assigns teacher-derived token importance weights during self-play fine-tuning~\citep{le2025token}, while CTPD transfers preference-aligned behavior across heterogeneous tokenizers by adapting token-level importance sampling~\citep{truong2026ctpd}. These methods use token-granular signals for weighting or teacher--student transfer, rather than deriving a per-prefix preference-optimal policy from pairwise comparisons.
More broadly, granularity-sensitive alignment has also been explored in teacher--student transfer: recent LLM distillation work aligns distributions, layer trajectories, or tokenizer-agnostic spans~\citep{hoang2026mcw,chi2026mta,dao2026sra}, while embedding distillation aligns relations and layers in compact representation spaces~\citep{truong2025emo,an2026mol,huy2026mipic}. Related topic-modeling work similarly shows that the choice of semantic representation space, including encoder and LLM-derived representations, affects topic quality and consistency~\citep{nguyen-etal-2025-xtra,DBLP:conf/aaai/PhatMVDN26,xuan2026llmxtmenhancingcrosslingualtopic}. We view this literature as complementary evidence that the unit of representation and alignment matters; TBPO studies this issue in preference learning by moving from sequence-level ratios to prefix- and token-level ratios.

\textbf{Density ratio matching.}
Another line of work interprets preference optimization through \emph{density ratio matching}.
Bregman Preference Optimization (BPO) formalizes this view by framing preference learning as likelihood-ratio estimation under Bregman divergences, connecting DPO to classical density-ratio estimators (e.g., logistic regression~\cite{JMLR:v13:gutmann12a}, KLIEP~\cite{NIPS2007_be83ab3e}, LSIF~\cite{lsif}) and showing that DPO is a special case of this framework \citep{kim2025preferenceoptimizationestimatingratio}.
BPO further introduces a scaled Basu's power divergence to interpolate between KLIEP and LSIF while modulating gradient reweighting during optimization.
In the safety-constrained setting, BSO derives safety alignment as density-ratio matching and introduces a Bregman family of single-stage safety objectives~\citep{nguyen2026bsosafetyalignmentdensity}.
Our focus is complementary: unlike these ratio-matching formulations, TBPO explicitly models token-level preferences and targets per-prefix optimality.

\textbf{Our contribution.}
Motivated by these gaps, we study \textbf{token-level preference modeling}. 
By defining preferences via a token-level Bradley--Terry model and adopting a ratio-matching perspective, we learn preference-optimal decisions at each prefix using only sequence-level comparison data.

\begin{figure*}[!hbt]
    \centering
    \includegraphics[width=1.0\linewidth]{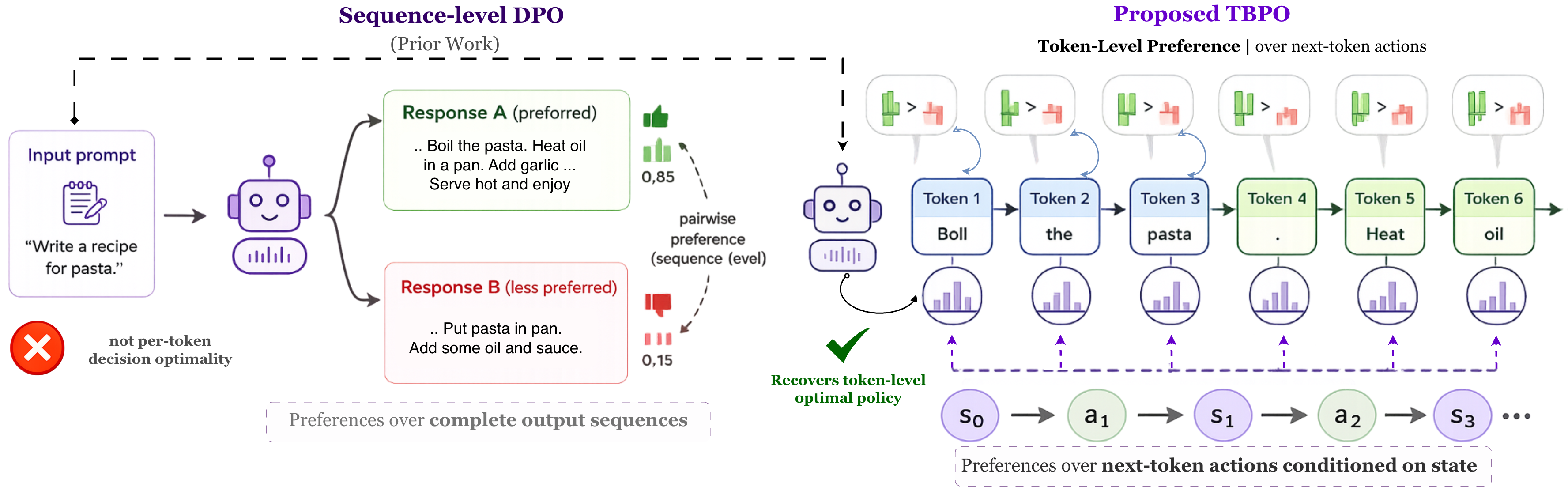}
    \caption{Overview comparison between our proposed TBPO and Sequence-level DPO.} 
    \label{fig:overview}
\end{figure*}

\section{Background}
\subsection{Preliminary and Notions}
When viewing text generation as a Markov decision process \citep{Puterman1994MDP}, we define the state at step $t$ as the prompt together with the response prefix produced so far, i.e., $s_t = [x, y_{<t}]$. The action is the next token to generate, $a_t = y_t$, and the per-token reward is given by $R_t := R(s_t, a_t) = R([x, y_{<t}], y_t)$.
Using these definitions, for a policy $\pi$ we define the state-action value function $Q^\pi$, the state value function $V^\pi$ and the advantage function $A^\pi$ as:

{\small
\begin{align*}
Q^\pi([x, y_{<t}], y_t) &= \mathbb{E}_\pi\left[\sum_{k=0}^\infty \gamma^k R_{t+k} \mid s_t = [x, y_{<t}], a_t = y_t\right], \nonumber \\
V^\pi([x, y_{<t}]) &= \mathbb{E}_\pi\left[Q^\pi([x, y_{<t}], y_t) \mid s_t = [x, y_{<t}]\right], \nonumber \\
A^\pi([x, y_{<t}], y_t) &= Q^\pi([x, y_{<t}], y_t) - V^\pi([x, y_{<t}]).
\end{align*}
}
where $\gamma$ denotes the discount factor. We set $\gamma = 1$ throughout this work.

Following TDPO \cite{zeng2024tokenleveldirectpreferenceoptimization}, we cast next-token generation as a sequential decision problem and optimize a KL-regularized objective at the token level:
\begin{equation}
\label{objective}
\begin{aligned}
\max_{\pi_{\theta}} \mathbb{E}_{x, y_{<t}\sim\mathcal{D},\, z\sim\pi_{\theta}(\cdot\mid[x,y_{<t}])}
\Bigl[ A_{\pi_{\text{ref}}}([x,y_{<t}],z) \\
\qquad - \beta\,D_{\mathrm{KL}}\!\Bigl(\pi_{\theta}(\cdot\mid[x,y_{<t}])\Vert\pi_{\text{ref}}(\cdot\mid[x,y_{<t}])\Bigr) \Bigr].
\end{aligned}
\end{equation}
Intuitively, each prefix state $[x,y_{<t}]$ defines a local decision over the next token. The advantage term encourages increasing the probability of tokens that are estimated (under the reference) to be locally better, while the KL penalty keeps the update conservative by anchoring $\pi_\theta$ to $\pi_{\mathrm{ref}}$. This is aligned with our goal of \textbf{token-level preference optimality}: the optimization objective explicitly targets preference-improving decisions at the state--action level, rather than only redistributing a sequence-level preference signal across timesteps. We later prove (Theorem \ref{optimality-theorem}) that, under sufficient model capacity, our token-level density-ratio matching objective recovers the optimal policy for Eq. \ref{objective}.

Using the objective in Eq.~\eqref{objective}, TDPO \citep{zeng2024tokenleveldirectpreferenceoptimization}
derives closed-form expressions for the optimal state-action $Q$ and advantage functions under the reference policy. 
\begin{equation}
\begin{split}
& Q_{\pi_{\mathrm{ref}}}([x,y_{<t}], z) = \\
& \qquad \beta \log \frac{\pi_{\theta}^*(z \mid [x,y_{<t}])}{\pi_{\mathrm{ref}}(z \mid [x,y_{<t}])} 
+ \beta \log Z([x,y_{<t}]; \beta),
\end{split}
\label{eq:tokenQformula}
\end{equation}
\begin{equation}
\small
\begin{split}
A_{\pi_{\mathrm{ref}}}([x,y_{<t}], z)
&= \beta \log \frac{\pi_{\theta}^*(z \mid [x,y_{<t}])}{\pi_{\mathrm{ref}}(z \mid [x,y_{<t}])} \\
&\hspace{-2em}
+ \beta D_{\mathrm{KL}}\!\Bigl(
\pi_{\mathrm{ref}}(\cdot\mid [x,y_{<t}])\,\big\|\,\pi_{\theta}^*(\cdot\mid [x,y_{<t}])
\Bigr).
\end{split}
\label{eq:tokenAformula}
\end{equation}
\noindent
where $Z([x,y_{<t}];\beta)
= \mathbb{E}_{z\sim\pi_{\mathrm{ref}}(\cdot\mid[x,y_{<t}])}
\, e^{\tfrac{1}{\beta}Q_{\pi_{\mathrm{ref}}}([x,y_{<t}],z)}$
is the partition function.

\subsection{Density Ratio Matching}
\label{sec:density_matching}
Consider two probability distributions $p_{\text{de}}(x)$ and $p_{\text{nu}}(x)$. The objective of likelihood ratio estimation is to construct a parametric model $R_\theta(x)$ that closely approximates the true ratio $R_{\text{data}}(x) := \frac{p_{\text{nu}}(x)}{p_{\text{de}}(x)}$ using independent and identically distributed samples drawn from each distribution. The predominant technique for this task is probabilistic classification through logistic regression~\cite{JMLR:v13:gutmann12a}. Beyond this, classical approaches including the Kullback-Leibler importance estimation procedure (KLIEP)~\cite{NIPS2007_be83ab3e} and least-squares importance fitting (LSIF)~\cite{lsif} have gained substantial adoption in practice.

A key insight from~\cite{Sugiyama2012DensityRatioMatching} is that these diverse ratio estimation techniques can be understood within a unified framework based on Bregman divergence~\cite{Bregman1967TheRM}. This unification leads to the following general objective:
\begin{equation*}
\begin{split}
D_h\!\left(R_{\text{data}}(x)\|R_\theta(x)\right)
&= \int p_{\text{de}}(x)\,
B_h\!\left(R_{\text{data}}(x)\|R_\theta(x)\right)\,dx .
\end{split}
\end{equation*}

which can be expanded as:
\begin{multline}
D_h\!\left(R_{\text{data}}(x)\|R_\theta(x)\right)
= \int p_{de}(x)\Bigl(
    h\!\left(R_{\text{data}}(x)\right)\\
  - h\!\left(R_{\theta}(x)\right) 
  - h'\!\left(R_{\theta}(x)\right)\bigl(R_{\text{data}}(x)-R_{\theta}(x)\bigr)
\Bigr)\,\mathrm{d}x .
\end{multline}
where $h$ represents a strictly convex and twice continuously differentiable function whose derivative is denoted by $h'$. The term $B_h$ corresponds to the pointwise Bregman divergence, which quantifies the approximation error of the first-order Taylor expansion.

\section{Methodology}
\label{sec:method}

A central modeling choice in preference optimization is the granularity at which preferences are defined. While standard DPO-style derivations posit a \textbf{sequence-level} Bradley–Terry model, generation and learning in LMs naturally operate at the \textbf{token level}. Our objective is therefore to formulate a token-level preference model, which enables density ratio matching at the token level and culminates in the framework \textbf{Token-level Bregman Preference Optimization (TBPO)}. We present in Figure~\ref{fig:overview} an overview of the proposed method.

\subsection{Token-level Bradley--Terry model}
\label{sec:token_bt}
For a preference triple $(x,y^w,y^l)$ and token step $t$, the two responses induce different prefixes
\begin{equation}
s_t^w = [x,y^w_{<t}],\qquad s_t^l = [x,y^l_{<t}],
\end{equation}
and the compared actions are $a_t^w=y_t^w$ and $a_t^l=y_t^l$.
We posit a token-level Bradley-Terry likelihood
\begin{equation}
 p_{\text{data}}(a_t^w \succ a_t^l \mid s_t^w,s_t^l)
=
\sigma\!\Big(S(s_t^w,a_t^w)-S(s_t^l,a_t^l)\Big),
\label{eq:bt_token}
\end{equation}
where $S(s,a)$ is a token-level score and the sigmoid function $\sigma(u)=1/(1+\exp(-u))$. Note that we call this preference probability as $p_{\text{data}}$ because  this is computed based on the preference data $y^w \succ y^l$ and we choose an appropriate score function to recover preference samples from data. Note that in DPO, $p_{\text{data}}\left(y^{w}\succ p^{l}\right)\coloneqq\sigma\left(r^{*}\left(x,y^{w}\right)-r^{*}\left(x,y^{l}\right)\right)$. 

The key idea is to lift a sequence-level preference label $y^w \succ y^l$ to supervision at each generation step. Conditioned on the prompt $x$ and the partial responses generated so far, $y^w_{<t}$ and $y^l_{<t}$, we prefer generating the next token from the winning trajectory, $y^w_t$, over the next token from the losing trajectory, $y^l_t$. This provides a preference signal throughout generation and encourages the extended prefixes $y^w_{\le t}$ and $y^l_{\le t}$ to separate in quality, as measured by a token-level score $S$.

Different choices of the score function $S$ yield different instantiations of the token-level Bradley--Terry model. We consider two natural options: the state-action value function $Q$ and the advantage function $A$, giving two variants: (1) TBPO-Q and (2) TBPO-A.

\subsection{TBPO-Q}
\label{sec:tbpo_q}

Here we instantiate the token-level Bradley--Terry model by using the reference-policy
state--action value as the score, i.e., $S(s,a)=Q_{\pi_{\mathrm{ref}}}(s,a)$:
\begin{equation} \label{eq:BT_Q}
\begin{split}
    &p_{\text{data}}(y^w_t \succ y_t^l \mid x, y^w_{<t}, y^l_{<t}) \\
    &= \sigma\bigl(Q_{\pi_{\mathrm{ref}}}([x, y^w_{<t}], y^w_t) - Q_{\pi_{\mathrm{ref}}}([x, y^l_{<t}], y^l_t)\bigr)
\end{split}
\end{equation}
\noindent\textbf{Intuition.}
Recall that $Q_{\pi_{\mathrm{ref}}}([x,y_{<t}],y_t)$ is the expected cumulative
reward obtained by emitting token $y_t$ at prefix $[x,y_{<t}]$ and then rolling out
$\pi_{\mathrm{ref}}$ for the remainder of the sequence.
Eq.~(\ref{eq:BT_Q}) therefore treats each preference label as a noisy comparison of two
one-step lookahead utilities: taking $y_t^w$ after the winning prefix versus taking $y_t^l$
after the losing prefix.
Equivalently, the Bradley--Terry log-odds are proportional to the difference in these expected
returns, so maximizing this likelihood pushes probability mass toward tokens whose downstream
continuations are expected to yield higher reward.

We can establish the following equality that links the optimal policy to the reference policy,
thereby later encouraging token-level density--ratio matching.
\begin{proposition}
\label{prop:q_token_ratio}
    Let $\pi_{\theta^*}$ denote the optimal policy for the optimization problem in Eq. (\ref{objective}). We have the following equality:
    \begin{equation}
    \begin{split}\frac{\pi_{\theta^*}\!\left(y_{t}^{w}\mid[x,y_{<t}^{w}]\right)}{\pi_{\theta^*}\!\left(y_{t}^{l}\mid[x,y_{<t}^{l}]\right)} & =\frac{\pi_{\mathrm{ref}}\!\left(y_{t}^{w}\mid[x,y_{<t}^{w}]\right)}{\pi_{\mathrm{ref}}\!\left(y_{t}^{l}\mid[x,y_{<t}^{l}]\right)}\\
\times & \left(\frac{p_{\text{data}}(y_{t}^{w}\succ y_{t}^{l}\mid x,y_{<t}^{w},y_{<t}^{l})}{p_{\text{data}}(y_{t}^{w}\prec y_{t}^{l}\mid x,y_{<t}^{w},y_{<t}^{l})}\right)^{\frac{1}{\beta}}w_{t},
\end{split}
    \end{equation}
    with per state weight $w_t = \exp\left(\log \frac{Z([x, y^l_{<t}]; \beta)}{Z([x, y^w_{<t}]; \beta)}\right)$.
\end{proposition}
See proof in Appendix~\ref{subsec:proof_q_token_ratio}.

\subsection{TBPO-A}
\label{sec:tbpo_a}

In this variants, we define a token-level BT preference model in terms of the
reference-policy advantage.
\begin{equation}
\label{eq:BT_A}
\begin{split}
    &p_{\text{data}}\!\left(y_t^w \succ y_t^l \mid x, y_{<t}^w, y_{<t}^l\right) \\
    &\quad = \sigma\!\Bigl(
    A_{\pi_{\mathrm{ref}}}\bigl([x, y^w_{<t}], y^w_t\bigr)
    - A_{\pi_{\mathrm{ref}}}\bigl([x, y^l_{<t}], y^l_t\bigr)
    \Bigr)
\end{split}
\end{equation}
This likelihood assigns higher probability to the token with larger advantage.
Optimizing Eq.~(\ref{objective}) therefore encourages the learned policy to
increase the advantage gap between the winning and losing extensions
$y^w_{\le t}$ and $y^l_{\le t}$, ensure generation quality at every state.

Similarly, we can establish the following equality that links the optimal policy to the reference policy, thereby encouraging token-level density--ratio matching. 
\begin{proposition}
\label{pro:a_token_ratio}
    Let $\pi_{\theta^*}$ denote the optimal policy for the optimization problem in Eq. (\ref{objective}). We have the following equality:
    \begin{equation}
    \begin{split}\frac{\pi_{\theta^*}\!\left(y_{t}^{w}\mid[x,y_{<t}^{w}]\right)}{\pi_{\theta^*}\!\left(y_{t}^{l}\mid[x,y_{<t}^{l}]\right)} & =\frac{\pi_{\mathrm{ref}}\!\left(y_{t}^{w}\mid[x,y_{<t}^{w}]\right)}{\pi_{\mathrm{ref}}\!\left(y_{t}^{l}\mid[x,y_{<t}^{l}]\right)}\\
\times & \left(\frac{p_{\text{data}}(y_{t}^{w}\succ y_{t}^{l}\mid x,y_{<t}^{w},y_{<t}^{l})}{p_{\text{data}}(y_{t}^{w}\prec y_{t}^{l}\mid x,y_{<t}^{w},y_{<t}^{l})}\right)^{\frac{1}{\beta}}w_{t},
\end{split}    
\end{equation}
    where we have defined per state weight
    \begin{equation*}
    \small
    \begin{split}
        w_t = \exp\Bigl( &D_{\mathrm{KL}}\bigl(\pi_{\mathrm{ref}}(\cdot\mid [x,y^l_{<t}]) \,\big\|\, \pi_{\theta}^{*}(\cdot\mid [x,y^l_{<t}])\bigr) \\
        &- D_{\mathrm{KL}}\bigl(\pi_{\mathrm{ref}}(\cdot\mid [x,y^w_{<t}]) \,\big\|\, \pi_{\theta}^{*}(\cdot\mid [x,y^w_{<t}])\bigr) \Bigr).
    \end{split}
    \end{equation*}
\end{proposition}
See proof in Appendix~\ref{subsec:proof_a_token_ratio}.

\subsection{Interpretation and estimation of per state weights}
In both Proposition~\ref{prop:q_token_ratio} (TBPO-Q) and
Proposition~\ref{pro:a_token_ratio} (TBPO-A), an extra multiplicative term $w_t$
appears in the policy ratio.
This term is \emph{state-dependent but action-independent}: it does not depend on
the compared next tokens $y_t^w,y_t^l$ directly, but only on the prefixes $[x, y^w_{<t}]$ and $[x, y^l_{<t}]$.
Intuitively, $w_t$ corrects for the fact that we compare actions taken in
\emph{two different states} (different prefixes), so state-only baseline terms
that would cancel in a same-state comparison no longer cancel.

\textbf{TBPO-Q: ratio of partition functions.}
The TBPO-Q weight is
\[
w_t^{(Q)}=\exp\left(\log \frac{Z([x, y^l_{<t}]; \beta)}{Z([x, y^w_{<t}]; \beta)}\right).
\]
therefore measures the \emph{relative continuation value} of the losing prefix
versus the winning prefix.
If the winning prefix is already a higher-value state (larger $Z([x,y^w_{<t}];\beta)$, then
$w_t^{(Q)}<1$, which downweights the token ratio update so that we do not
over-attribute preference to the single next-token decision when the prefixes
already differ substantially in quality.

\textbf{TBPO-A: difference of KL baselines.}
The TBPO-A weight is
\begin{equation*}
\begin{split}
    w_t^{(A)} = \exp\Bigl( &D_{\mathrm{KL}}\bigl(\pi_{\mathrm{ref}}(\cdot\mid [x,y^l_{<t}]) \,\big\|\, \pi_{\theta}^{*}(\cdot\mid [x,y^l_{<t}])\bigr) \\
    &- D_{\mathrm{KL}}\bigl(\pi_{\mathrm{ref}}(\cdot\mid [x,y^w_{<t}]) \,\big\|\, \pi_{\theta}^{*}(\cdot\mid [x,y^w_{<t}])\bigr) \Bigr),
\end{split}
\end{equation*}
which can be interpreted as a \emph{relative policy-improvement cost} between the two
prefixes: it compares how far the optimal policy must deviate from the reference
at $[x, y^l_{<t}]$ versus $[x, y^w_{<t}]$ under the KL regularizer.

\subsubsection{Practical estimation.}
In practice $Z([x, y_{<t}]; \beta)$ and $D_{\mathrm{KL}}$ are either computationally intractable or expensive. We use plug-in estimators that preserve
the key property that $w_t$ is state-only.

\textbf{TBPO-Q (learned baseline).}
We parameterize a scalar baseline head $b_\phi([x,y_{<t}])$ to predict
$\log Z([x,y_{<t}];\beta)$ and set
\[
\begin{aligned}
\widehat{\log w_t^{(Q)}} 
&= b_\phi\big([x, y^l_{<t}]\big)
 - b_\phi\big([x, y^w_{<t}]\big), \\
\hat w_t^{(Q)} 
&= \exp\!\left(\widehat{\log w_t^{(Q)}}\right).
\end{aligned}
\]
In practice, we attach a lightweight MLP head on top of the LM's final hidden layer to predict $Z([x, y_{<t}]; \beta)$. We train this head during optimization while stopping gradients into the LM backbone, so the LLM does not allocate capacity to baseline estimation at the expense of preference optimization; see Appendix~\ref{imp_details}.

\textbf{TBPO-A (K3 KL estimator).}
We approximate $\pi^*$ by the current policy $\pi_\theta$ and estimate
$D_{\mathrm{KL}}(\pi_{\mathrm{ref}}(\cdot\mid s)\,\|\,\pi_\theta(\cdot\mid s))$
using the low-variance unbiased estimator~\citep{Schulman2020ApproximatingKL}. See more in the Appendix \ref{imp_details}.

\paragraph{TBPO-Q vs.\ TBPO-A.}
Both variants require a state-only correction because token comparisons are made across two
different prefixes.
TBPO-Q’s $w_t$ is a \emph{partition-function / soft-value} ratio and is naturally handled
by a learned baseline head.
TBPO-A removes the partition function analytically but introduces a \emph{KL-baseline} term,
which can be estimated directly from $\pi_{\mathrm{ref}}$ and $\pi_\theta$ using an unbiased estimator.

\newcommand{\mstd}[2]{#1$_{\pm #2}$}
\newcommand{\best}[1]{\textbf{#1}}
\newcommand{\second}[1]{\underline{#1}}

\begin{table*}[t]
\centering
\small
\setlength{\tabcolsep}{7pt}
\renewcommand{\arraystretch}{1.15}
\resizebox{\textwidth}{!}{%
\begin{tabular}{lccccccc}
\toprule
\textbf{Method} & \textbf{HellaSwag} & \textbf{ARC} & \textbf{MMLU} & \textbf{TruthfulQA} & \textbf{Winogrande} & \textbf{GSM8K} & \textbf{Average} \\
\midrule

\rowcolor{gray!15}
\multicolumn{8}{c}{\textbf{Mistral 7B v0.1}} \\
\midrule
SFT
& \mstd{80.72}{0.39} & \mstd{55.54}{1.45} & \mstd{58.41}{0.39} & \mstd{43.67}{1.46} & \mstd{76.71}{1.18} & \mstd{18.49}{1.07} & 55.59 \\
DPO
& \mstd{82.39}{0.38} & \mstd{59.40}{1.43} & \mstd{59.15}{0.38} & \mstd{42.30}{1.49} & \mstd{77.74}{1.16} & {\mstd{34.87}{1.21}} & 59.30 \\
TDPO
& \mstd{82.70}{0.38} & \best{\mstd{61.54}{1.41}} & \mstd{57.17}{0.40} & \mstd{43.48}{1.53} & \mstd{76.95}{1.18} & \mstd{28.65}{1.24} & 58.41 \\
TIS-DPO
& \best{\mstd{83.10}{0.37}} & \mstd{59.09}{1.42} & \mstd{57.81}{0.39} & \best{\mstd{45.91}{1.56}} & \mstd{77.66}{1.17} & \mstd{34.42}{1.31} & 59.66 \\
BPO
& \mstd{81.03}{0.39} & \mstd{57.67}{1.44} & \mstd{59.16}{0.40} & \mstd{41.71}{1.46} & \mstd{77.26}{1.16} & \mstd{32.90}{1.29} & 58.29 \\
\addlinespace[3pt]
TBPO-Q (Ours)
& \mstd{82.72}{0.37} & \second{\mstd{59.74}{1.41}} & \best{\mstd{60.82}{0.39}} & \second{\mstd{44.47}{1.45}} & {\best{\mstd{78.58}{1.11}}} & \best{\mstd{39.34}{1.34}} & \best{60.95} \\
TBPO-A (Ours)
& \second{\mstd{82.74}{0.38}} & \mstd{59.47}{1.43} & \second{\mstd{60.54}{0.39}} & \mstd{44.40}{1.46} & \second{\mstd{78.21}{1.21}} & \second{\mstd{39.04}{1.35}} & \second{60.73} \\

\midrule
\rowcolor{gray!15}
\multicolumn{8}{c}{\textbf{Llama3 8B}} \\
\midrule
SFT
& \mstd{79.46}{0.40} & \mstd{51.45}{1.46} & \mstd{61.19}{0.39} & \mstd{46.46}{1.45} & \mstd{75.69}{1.21} & \mstd{68.67}{1.37} & 63.82 \\
DPO
& {\mstd{81.23}{0.38}} & \mstd{56.05}{1.45} & \mstd{61.70}{0.38} & \mstd{48.47}{1.53} & \mstd{75.84}{1.20} & \mstd{74.20}{1.37} & 66.24 \\
TDPO
& \best{\mstd{83.29}{0.37}} & {\mstd{59.04}{1.43}} & {\mstd{61.86}{0.39}} & {\mstd{51.76}{1.55}} & \mstd{76.55}{1.19} & {\mstd{75.10}{1.36}} & {67.91} \\
TIS-DPO
& \mstd{81.37}{0.38} & \mstd{58.70}{1.42} & \mstd{60.77}{0.39} & \mstd{49.87}{1.57} & \mstd{75.14}{1.21} & \mstd{75.12}{1.37} & 66.82 \\
BPO
& \mstd{81.68}{0.37} & \mstd{54.95}{1.45} & \mstd{61.48}{0.38} & \mstd{47.88}{1.52} & \mstd{75.13}{1.21} & \mstd{73.22}{1.37} & 65.72 \\
\addlinespace[3pt]
TBPO-Q (Ours)
& \mstd{81.90}{0.38} & \second{\mstd{64.07}{1.40}} & \best{\mstd{64.92}{0.38}} & \best{\mstd{53.36}{1.52}} & \second{\mstd{79.08}{1.14}} & \second{\mstd{78.92}{1.12}} & \second{70.38} \\
TBPO-A (Ours)
& \second{\mstd{81.97}{0.38}} & \best{\mstd{64.08}{1.40}} & \second{\mstd{64.84}{0.39}} & \second{\mstd{53.29}{1.53}} & \best{\mstd{79.63}{1.12}} & \best{\mstd{79.22}{1.21}} & \best{70.50} \\

\bottomrule
\end{tabular}%
}
\vspace{0.05in}
\caption{General capabilities on the Open LLM Leaderboard suite for Mistral~7B~v0.1 and Llama~3~8B. Scores are reported by \texttt{lm-eval-harness} (higher is better) with $\pm$ standard error. Best is in \textbf{bold}, second-best is \underline{underlined}.}
\vspace{-0.15in}
\label{tab:general_style_like_example}
\end{table*}

\subsection{Token-level Bregman ratio matching}
\label{sec:bregman}
We now use density-ratio matching to derive a tractable optimization objective from the token-level BT preference model.
Given a time-step $t$, we could naturally define the BT preference model for the sequences up to the $t$-th token:
\begin{equation}
p_{\text{data}}\left(y_{\leq t}^{w}\succ y_{\leq t}^{l}\mid x\right)=\prod_{i=1}^{t}p_{\text{data}}\left(y_{i}^{w}\succ y_{i}^{l}\mid x,y_{<i}^{w},y_{<i}^{l}\right).\label{eq:uptot_pref}
\end{equation}

It is evident that we have the recursive formula:
\begin{align}
p_{\text{data}}\left(y_{\leq t}^{w}\succ y_{\leq t}^{l}\mid x\right) & =p_{\text{data}}\left(y_{t}^{w}\succ y_{t}^{l}\mid x,y_{<t}^{w},y_{<t}^{l}\right)\nonumber \\
 & p_{\text{data}}\left(y_{\leq t-1}^{w}\succ y_{\leq t-1}^{l}\mid x\right).\label{eq:recursive}
\end{align}

Preference optimization can be reformulated as a matching problem between the empirical data ratio  
$R^t_{\text{data}}(y^w_t, y^l_t \mid y^w_{<t}, y^l_{<t}, x)$  
and the model ratio  
$R^t_{\theta}(y^w_t, y^l_t \mid y^w_{<t}, y^l_{<t}, x)$, where we have defined
\begin{gather}
R_{\text{data}}^{t}=\frac{p_\text{data}(y_{t}^{w}\prec y_{t}^{l}\mid x,y_{<t}^{w},y_{<t}^{l})}{p_\text{data}(y_{t}^{w}\succ y_{t}^{l}\mid x,y_{<t}^{w},y_{<t}^{l})},\nonumber \\
R_{\theta}^{t}=\left(\frac{\pi_{\theta}\left(y_{t}^{l}\mid[x,y_{<t}^{l}]\right)\pi_{\mathrm{ref}}\!\left(y_{t}^{w}\mid[x,y_{<t}^{w}]\right)}{\pi_{\mathrm{ref}}\!\left(y_{t}^{l}\mid[x,y_{<t}^{l}]\right)\pi_{\theta}\!\left(y_{t}^{w}\mid[x,y_{<t}^{w}]\right)}w_{t}\right)^{\beta}.\label{eq:token_matching}
\end{gather}
We define the $t$-th time step loss using the Bregman ratio matching framework in Section \ref{sec:density_matching} as:
\footnotesize{
\begin{multline*}
D_{h}^{t}\left(R_{\text{data}}^{t},R_{\theta}^{t}\right)=\\
\mathbb{E}_{p_{\text{data}}\left(y_{\leq t}^{w}\succ y_{\leq t}^{l}\mid x\right)}\left[h\left(R_{\text{data}}^{t}\right)-h\left(R_{\theta}^{t}\right)-h^{'}\left(R_{\theta}^{t}\right)\left(R_{\text{data}}^{t}-R_{\theta}^{t}\right)\right].
\end{multline*}
}
\normalsize
The final loss is the average of the component losses:
\begin{equation}
D_{h}\left(R_{\text{data}},R_{\theta}\right)=\frac{1}{T}\sum_{t=1}^{T}D_{h}^{t}\left(R_{\text{data}}^{t},R_{\theta}^{t}\right),\label{eq:total_loss}
\end{equation}
where $T=\min(|y^w|,|y^l|)$ is the length of the shorter response among the chosen and rejected sequences.

In the following theorem, we show that under the sufficient model capacity, by matching the density ratio at the token level, we can recover the optimal policy in the Objective \ref{objective}.
\begin{theorem}
\label{optimality-theorem}
    Under sufficient model capacity, we have $\pi_{\theta^{*}}=\pi_{\hat{\theta}}$ where $\pi_{\hat{\theta}}=\text{argmin}{}_{\pi_{\theta}}\,D_{h}\left(R_{\text{data}},R_{\theta}\right)$.
\end{theorem}

The loss function $D_{h}\left(R_{\text{data}},R_{\theta}\right)$ is not tractable to train a policy model. We turn it to an equivalent tractable objective function in the following theorem.

\begin{theorem}
\label{theorem:loss}
    We have $\mathcal{L}\left(\theta,p_{\text{data}}\right)=D_{h}\left(R_{\text{data}},R_{\theta}\right)+\text{const}$ where we define
    {\footnotesize
\begin{equation}
\begin{split}
\mathcal{L}\left(\theta,p_{\text{data}}\right)
&=
\mathbb{E}_{p_{\text{data}}\left(y^{w}\succ y^{l}\mid x\right)}
\Biggl[
\frac{1}{T}\sum_{t=1}^{T}
\Bigl(
h'\left(R_{\theta}^{t}\right)R_{\theta}^{t}
\\
&\qquad
- h\left(R_{\theta}^{t}\right)
- h'\left(\left(R_{\theta}^{t}\right)^{-1}\right)
\Bigr)
\Biggr]
\end{split}
\label{eq:loss}
\end{equation}}

\end{theorem}
The loss in Eq.~\eqref{eq:loss} depends only on the ratio $R_{\theta}^t$, hence applies identically to TBPO-Q and TBPO-A.
Different ratio estimation methods differ only in their choice of function $h$. Following BPO \cite{kim2025preferenceoptimizationestimatingratio}, we employ $h$ function as Scaled Basu’s power divergence (SBA), with following formula $h(R) = \frac{R^{1+\lambda} - R}{s\lambda(\lambda+1)}$ for better training flexibility and tuning flexibility. The $\lambda$ parameter lets you control whether to focus on confident samples (low $R_{\theta}$) or uncertain ones ($R_{\theta} \approx 1$), which can be useful for handling noisy preference data.



\subsection{Overall principle of our approach}
The standard Bradley--Terry (BT) model encourages the entire sequence $y^w \succ y^l$.
In contrast, we propose a token-level BT model that encourages progressively better substrings $y^w_{\le t} \succ y^l_{\le t}$ as $t$ varies according to the recursive definition in Eq.~\eqref{eq:recursive}.
Our goal is not to treat generation only as a whole-sequence preference problem, but to encourage the model to form good partial reasoning early and maintain it throughout generation.
In this way, stronger intermediate prefixes can support further high-quality continuation, which helps address credit assignment in a more fine-grained way and can lead to more concise, shorter, and more precise responses.

From this perspective, sequence-level preference labels are not merely copied to every timestep as hard tokenwise win/loss annotations.
Instead, they induce a sequence of prefix-conditioned local comparisons, which is exactly the structure captured by our token-level BT formulation.
This is the key difference from prior approaches whose underlying preference model remains sequence-level and mainly redistributes the same global signal across positions.
Our formulation instead makes the preference modeling itself token-level, so that better local decisions are encouraged progressively during autoregressive generation.

We also note that a naive decomposition could overemphasize locally plausible but globally suboptimal decisions.
TBPO is designed to mitigate exactly this issue.
First, supervision is defined over progressively extended prefixes, not isolated tokens, which preserves sequential dependency.
Second, the state-dependent correction weight reduces mismatch when comparing next-token decisions made under different prefix contexts, preventing the model from attributing the entire sequence-level preference gap to a single step.
In this sense, TBPO should be viewed as a principled subsequence-level surrogate derived from sequence-level preferences, together with an explicit mechanism to reduce the local--global mismatch, rather than as an assumption that improving local token decisions is always sufficient by itself.

\begin{figure*}[t] 
  \centering
  \includegraphics[width=0.98\textwidth,page=1]{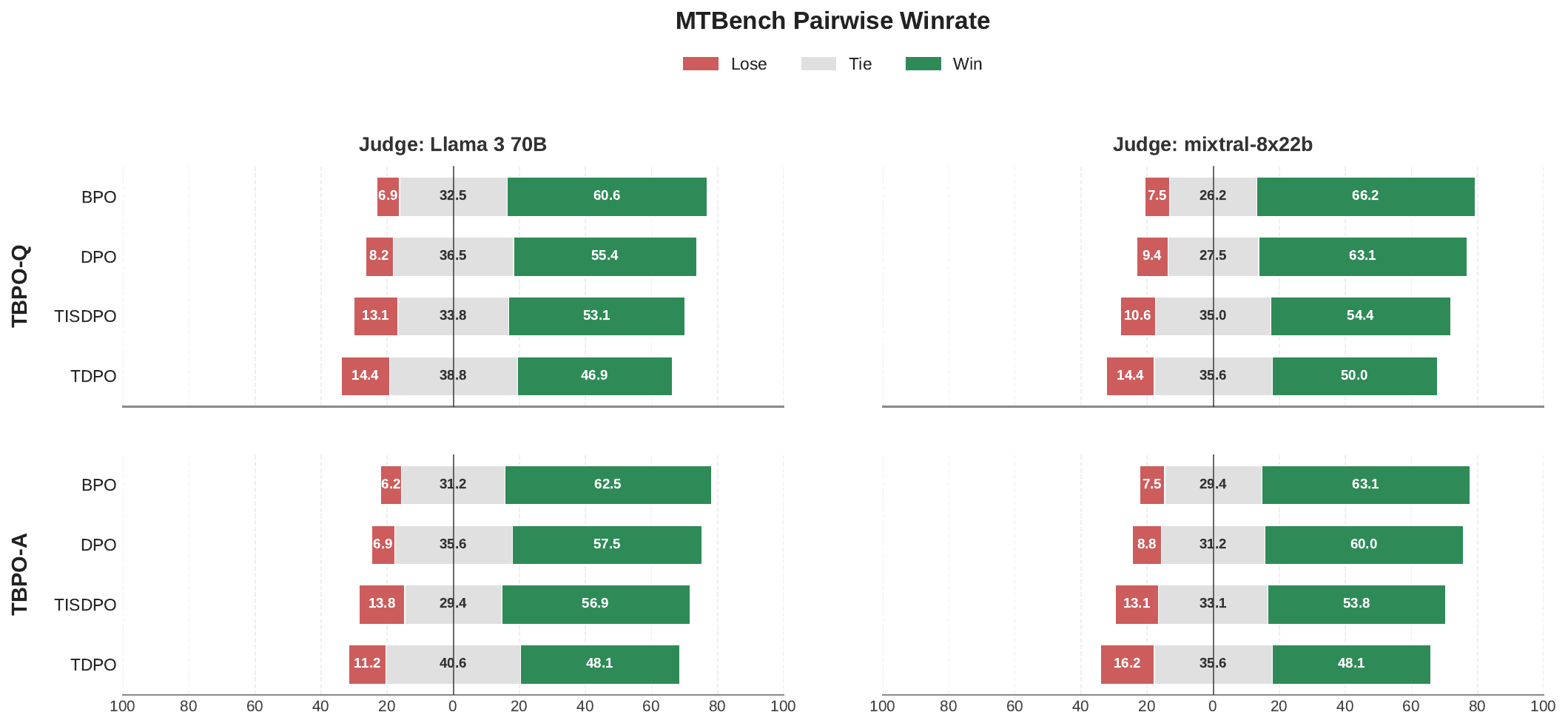}
  \caption{MT-Bench pairwise win/tie/lose rates for TBPO-Q (top) and TBPO-A (bottom) against prior preference-optimization baselines, evaluated by two LLM judges. TBPO achieves higher win rates with low loss rates across both judges, and the advantage persists even against the strongest baseline.}
  \label{fig:mtbench}
  \vspace{-0.1in}
\end{figure*}

\section{Experiments}
\label{sec:experiments}

In this section, we present an empirical study of Token-level Bregman Preference Optimization (TBPO). We compare TBPO-Q and TBPO-A to existing preference-optimization methods, with particular attention to alignment quality, training stability, and generation diversity, evaluating whether the theoretical advantages of token-level ratio matching manifest in practice.
\subsection{Experiments settings}

\paragraph{Backbones and training data.}
We conduct experiments with two backbone models. Each setup initializes from a publicly available supervised fine-tuning (SFT) checkpoint, followed by preference optimization on the corresponding dataset. Specifically, Mistral 7B v0.1 is initialized from \href{https://huggingface.co/HuggingFaceH4/mistral-7b-sft-alpha}{mistral-7b-sft-alpha} and trained on \href{https://huggingface.co/datasets/HuggingFaceH4/ultrafeedback_binarized}{UltraFeedback Binarized}, while Llama 3 8B is initialized from \href{https://huggingface.co/RLHFlow/LLaMA3-SFT-v2}{LLaMA3-SFT-v2} and trained on \href{https://huggingface.co/datasets/princeton-nlp/llama3-ultrafeedback-armorm}{Llama3-UltraFeedback-ArmoRM}.
\vspace{-0.1in}
\paragraph{Baselines.}
We compare TBPO-Q and TBPO-A against various prior preference-optimization baselines, including SFT, DPO \cite{rafailov2023dpo}, TDPO \cite{zeng2024tokenleveldirectpreferenceoptimization}, TIS-DPO \cite{tis-dpo}, and BPO-SBA \cite{kim2025preferenceoptimizationestimatingratio}. See more details in Appendix~\ref{app:baselines}.

\vspace{-0.1in}
\paragraph{LLM judges.}
Some benchmarks, including Anthropic HH-RLHF, TL;DR, and MT-Bench, require pairwise evaluation with an LLM-as-a-judge framework. Prior work has analyzed the reliability and biases of LLM judges~\cite{zhang2025compassjudger2generalistjudgemodel, han2025judgesverdictcomprehensiveanalysis}; following these findings, we adopt task-specific judge pairs. For Anthropic HH-RLHF and TL;DR, we use \href{https://huggingface.co/meta-llama/Meta-Llama-3-70B-Instruct}{Llama 3 70B Instruct} and \href{https://huggingface.co/deepseek-ai/DeepSeek-V3}{DeepSeek-V3}, while for MT-Bench we use \href{https://huggingface.co/meta-llama/Meta-Llama-3-70B-Instruct}{Llama 3 70B Instruct} and \href{https://huggingface.co/mistralai/Mixtral-8x22B-Instruct-v0.1}{Mixtral-8x22B Instruct v0.1}. Additional details are provided in Appendix~\ref{app:experiments}.

\begin{figure*}[t] 
  \centering
  \includegraphics[width=0.9\textwidth,page=1]{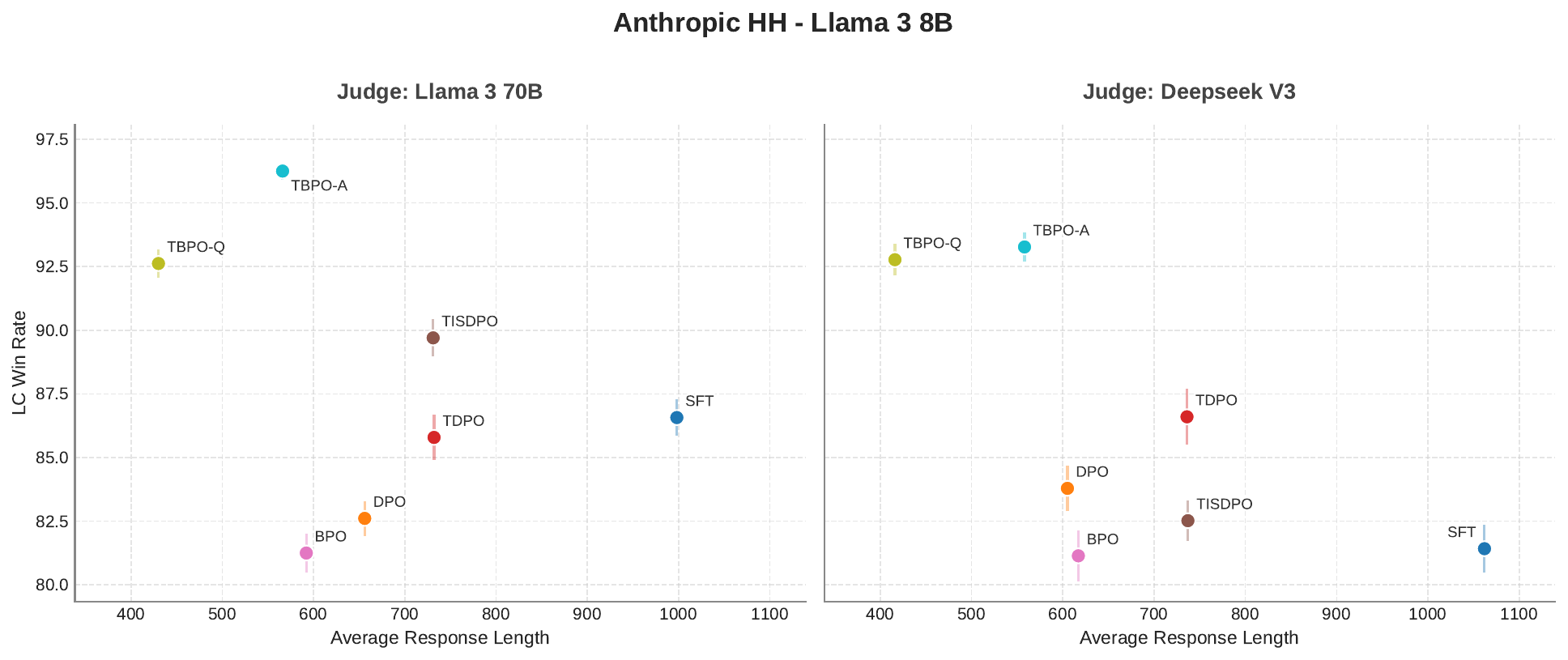}
  \caption{LC win rate vs.\ average response length against the dataset-preferred completion for Llama~3~8B, evaluated by two LLM judges (Llama~3~70B, DeepSeek-V3); error bars are $\pm$1 s.e.\ over 200 prompts. TBPO leads with shorter outputs, indicating gains beyond verbosity and consistent across judges.}
  \vspace{-0.15in}
  \label{fig:hh}
\end{figure*}

\subsection{General Capabilities}
A common concern in preference optimization is over-specialization to preference data, which may harm general capabilities. 
We evaluate general performance using the \textbf{HuggingFace Open LLM Leaderboard} protocol~\cite{huggingface_open_llm_leaderboard_v1} with the Language Model Evaluation Harness~\cite{eleutherai_lm_eval_harness_concept} on six benchmarks spanning commonsense reasoning (ARC~\cite{clark2018think}, HellaSwag~\cite{zellers2019hellaswag}, Winogrande~\cite{sakaguchi2021winogrande}), multi-task understanding (MMLU~\cite{hendrycks2020mmlu}), factuality (TruthfulQA~\cite{lin2021truthfulqa}), and math reasoning (GSM8k~\cite{cobbe2021gsm8k}). 
Table~\ref{tab:general_style_like_example} shows that TBPO improves the overall average across both backbones while remaining competitive on all benchmarks, indicating stronger alignment without sacrificing core capabilities.

On \textbf{Mistral 7B v0.1}, sequence-level preference optimization already improves over SFT (DPO: 59.30 vs.\ SFT: 55.59). TBPO further improves the average: \textbf{TBPO-Q is best} (\textbf{60.95}) with TBPO-A close behind (60.73), exceeding the strongest non-TBPO baseline (TIS-DPO: 59.66). Gains are broad and strongest on reasoning-heavy tasks: TBPO-Q attains the best GSM8K (39.34 vs.\ 34.87 for DPO and 18.49 for SFT) and also the best MMLU (60.82) and Winogrande (78.58). While DPO slightly reduces TruthfulQA vs.\ SFT (42.30 vs.\ 43.67), TBPO restores and improves it (44.47--44.40). Where other token-level methods peak (e.g., ARC for TDPO, TruthfulQA for TIS-DPO), TBPO remains highly competitive.

On \textbf{Llama 3 8B}, \textbf{TBPO-A is best on average} (\textbf{70.50}) with TBPO-Q close behind (70.38), improving substantially over TDPO (67.91). Improvements concentrate on ARC (+5 vs.\ TDPO; 64.08 vs.\ 59.04), GSM8K (+4; 79.22 vs.\ 75.10), and Winogrande (+3; 79.63 vs.\ 76.55). TBPO-A leads on ARC/GSM8K/Winogrande, while TBPO-Q attains the best MMLU (64.92) and TruthfulQA (53.36), showing robustness to the preference-model formulation.

The similar performance of TBPO-Q and TBPO-A is expected from the theory: the two variants differ mainly in how they approximate the same state-only correction term $w_t$, rather than in the core token-level preference objective. TBPO-Q represents $w_t$ as a partition-function or soft-value ratio and estimates it with a learned baseline head, whereas TBPO-A analytically removes the partition function and replaces it with a KL-based state baseline estimated directly from $\pi_{\mathrm{ref}}$ and $\pi_\theta$. Since both estimators target the same prefix-level correction, their optimization targets remain nearly identical when the estimators are stable. The near ties in Table~\ref{tab:general_style_like_example} therefore suggest that the main gain comes from the shared token-level ratio-matching formulation, while the particular estimator used for $w_t$ is a secondary implementation choice.

Overall, TBPO delivers the strongest mean performance across two model families, with large gains on reasoning-centric benchmarks and no evidence of systematic trade-offs on broader language understanding. To better understand these reasoning gains, we provide qualitative case studies on reasoning questions in Appendix~\ref{app:reasoning_case_study}.

\subsection{Generation Quality}

\paragraph{MT-Bench.} To evaluate generation quality, we run MT-Bench~\cite{zheng2023judging} under the standard LLM-as-a-judge protocol and report pairwise win/tie/lose rates against prior preference-optimization baselines. Figure~\ref{fig:mtbench} shows that TBPO-Q and TBPO-A consistently improve win rates with low loss rates; against the strongest baseline (TDPO), many comparisons are ties, but TBPO wins the majority of decisive (non-tie) cases. This suggesting the advantage is not driven by a small number of outliers. The overall ordering and effect sizes are consistent across both judges, supporting the conclusion that TBPO improves generation quality in a robust, judge-agnostic manner, and that TBPO-Q and TBPO-A are comparably strong.

\paragraph{Harmless and Helpfulness.}


To evaluate preference alignment in safety-sensitive dialogue, we test on 200 held-out samples from the \href{https://huggingface.co/datasets/Anthropic/hh-rlhf}{Anthropic HH-RLHF} dataset. 
Using an LLM-as-a-judge protocol, we report pairwise and length-controlled (LC) win rates~\cite{alpaca_eval}. 
Figure~\ref{fig:hh} shows LC win rate versus average response length, highlighting the verbosity--quality trade-off.

Across both judges, TBPO variants achieve the strongest preference alignment while producing shorter responses than competing methods, ruling out verbosity as the primary driver of higher win rates. SFT is the most verbose model yet does not achieve top performance, whereas TBPO attains the best LC win rates with substantially shorter generations. Among TBPO variants, TBPO-Q produces the most concise responses while remaining at the top of the win-rate plot, and TBPO-A achieves the highest (or near-highest) win rates with only a modest increase in length, indicating a controllable conciseness–performance trade-off. In contrast, strong sequence- and token-level baselines improve over SFT but generally require longer responses and still trail TBPO. The consistent ordering across judges and large margins relative to standard errors suggest TBPO’s gains reflect genuine improvements in helpful-and-harmless response quality rather than judge bias or verbosity effects.
\vspace{-0.1in}
\paragraph{Generation Diversity.}
\begin{figure}[t]
    \centering
    \includegraphics[width=0.9\linewidth,page=1]{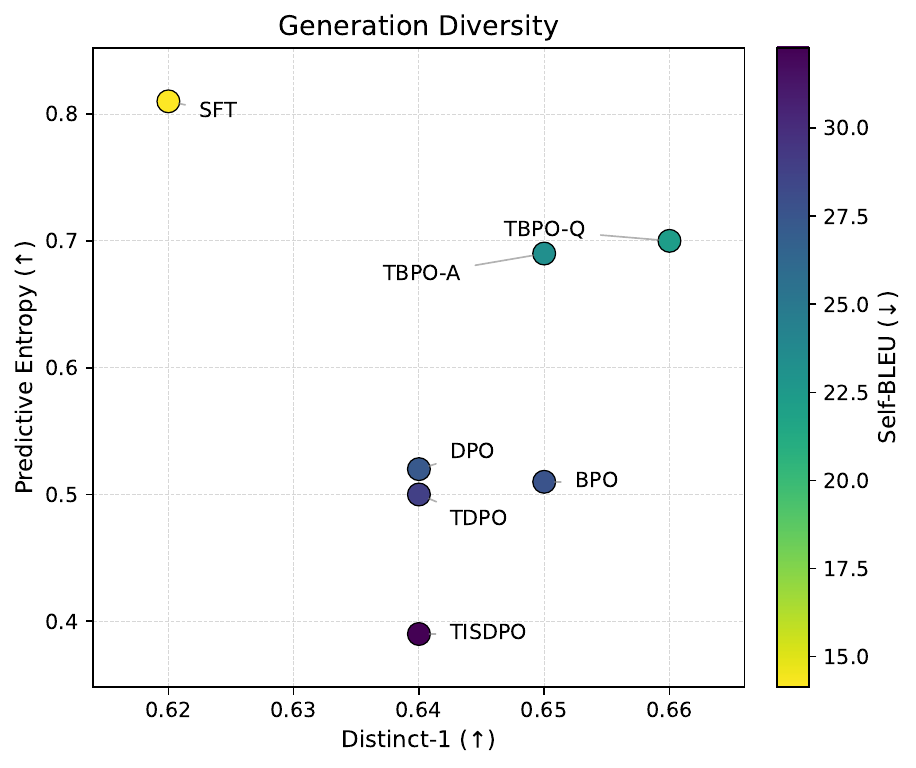}
    \caption{Generation diversity trade-offs: Distinct-1 vs. predictive entropy (higher is better), colored by self-BLEU (lower is better). TBPO achieves the best three-way trade-off across metrics.}
    \label{fig:diversity}
    \vspace{-0.2in}
\end{figure}
Preference optimization can reduce generation diversity~\cite{rlhf-diversity}. We evaluate diversity on 100 held-out prompts from the test split of \href{https://huggingface.co/datasets/HuggingFaceH4/ultrafeedback_binarized}{UltraFeedback Binarized} using predictive entropy~\cite{zeng2024tokenleveldirectpreferenceoptimization}, Self-BLEU~\cite{self-bleu}, and Distinct-1~\cite{distinct-1}. Figure~\ref{fig:diversity} summarizes the trade-off, where higher Distinct-1 and predictive entropy indicate greater diversity, while lower Self-BLEU indicates less redundancy.

Two patterns emerge. (1) DPO/TDPO/BPO have moderate Distinct-1 but lower predictive entropy and higher self-BLEU, implying samples become more similar and decoding more deterministic (mild mode-seeking). TISDPO amplifies this, achieving the lowest entropy and highest self-BLEU, i.e., the strongest collapse even when Distinct-1 alone does not fully reveal it. (2) TBPO shifts the Pareto frontier toward the desirable upper-right region while keeping self-BLEU favorable, giving the best three-metric trade-off: versus DPO-style methods it raises entropy and Distinct-1 without the usual across-sample diversity loss; versus SFT (high entropy, low self-BLEU, lower Distinct-1) it preserves non-collapsed sampling while increasing lexical variety. Overall, TBPO better balances preference satisfaction with diversity preservation, producing outputs that are less redundant across samples and less over-confident at the token level.

Additional experiments on out-of-domain tasks, ability to handle noisy data, ablation on per-state weight estimation and judge agreement analysis for TBPO are reported in Appendix~\ref{app:additional}.

\section{Conclusion}
\label{sec:conclusion}
We presented a principled token-level preference optimization framework aligned with the autoregressive nature of language models. We show theoretically that minimizing the proposed objective recovers a token-wise optimal policy, establishing a direct link between preference modeling and optimization behavior. Experiments across benchmarks and model families demonstrate consistent improvements over strong baselines, particularly in challenging reasoning and alignment settings.


\section*{Acknowledgements}
Trung Le was supported by the Air Force Office of Scientific Research under award number FA2386-25-1-4023 and the ARC Discovery Project grant DP250100262. Linh Ngo Van is funded by Vietnam National Foundation for Science and Technology
Development (NAFOSTED) under grant number 102.05-2025.16.

\section*{Impact Statement}
This work focuses on the development of a principled learning algorithm with theoretical guarantees. The primary impact of this research is methodological, contributing to the design of more reliable and well-understood machine learning systems. While the proposed method is evaluated on standard benchmarks, it is not tied to a specific high-risk application domain. We do not foresee any immediate negative societal consequences arising directly from this work; however, as with most advances in machine learning, downstream applications should be deployed responsibly and with appropriate consideration of ethical and societal implications.

\nocite{langley00}

\bibliography{ref}
\bibliographystyle{icml2026}

\newpage
\appendix
\onecolumn

\section{Theory Development}
\label{app:theory}

\subsection{Proof for Prop \ref{prop:q_token_ratio}}
\label{subsec:proof_q_token_ratio}
\begin{proof}
We prove the proposition by deriving the policy ratio induced by a
token-level Bradley--Terry model parameterized by the reference-policy
$Q$-function.

By the token-level Bradley--Terry model in Eq.~\eqref{eq:BT_Q}, we have
\begin{equation*}
p(y_t^w \succ y_t^l \mid x, y^w_{<t}, y^l_{<t})
=
\sigma\!\Big(
Q_{\pi_{\mathrm{ref}}}([x,y^w_{<t}], y_t^w)
-
Q_{\pi_{\mathrm{ref}}}([x,y^l_{<t}], y_t^l)
\Big).
\end{equation*}
Using $\sigma(u)=(1+\exp(-u))^{-1}$ and conditioning on the same pair of
prefixes, taking the odds ratio yields
\[
\frac{
p(y_t^w \prec y_t^l \mid x, y^w_{<t}, y^l_{<t})
}{
p(y_t^w \succ y_t^l \mid x, y^w_{<t}, y^l_{<t})
}
=
\exp\!\Big(
Q_{\pi_{\mathrm{ref}}}([x,y^l_{<t}], y_t^l)
-
Q_{\pi_{\mathrm{ref}}}([x,y^w_{<t}], y_t^w)
\Big).
\]

Substituting the closed-form expression of the reference-policy
$Q$-function from Eq.~\eqref{eq:tokenQformula} gives
\[
\begin{aligned}
\frac{
p(y_t^w \prec y_t^l \mid x, y^w_{<t}, y^l_{<t})
}{
p(y_t^w \succ y_t^l \mid x, y^w_{<t}, y^l_{<t})
}
&=
\exp\Bigg(
\beta \log \frac{\pi^*_\theta(y_t^l \mid [x,y^l_{<t}])}
{\pi_{\mathrm{ref}}(y_t^l \mid [x,y^l_{<t}])}
-
\beta \log \frac{\pi^*_\theta(y_t^w \mid [x,y^w_{<t}])}
{\pi_{\mathrm{ref}}(y_t^w \mid [x,y^w_{<t}])}
\\
&\qquad\quad
+
\beta \log Z([x,y^l_{<t}]; \beta)
-
\beta \log Z([x,y^w_{<t}]; \beta)
\Bigg).
\end{aligned}
\]

Rearranging terms, we obtain
\[
\frac{
p(y_t^w \prec y_t^l \mid x, y^w_{<t}, y^l_{<t})
}{
p(y_t^w \succ y_t^l \mid x, y^w_{<t}, y^l_{<t})
}
=
\left(
\frac{
\pi^*_\theta(y_t^l \mid [x,y^l_{<t}])\,
\pi_{\mathrm{ref}}(y_t^w \mid [x,y^w_{<t}])
}{
\pi_{\mathrm{ref}}(y_t^l \mid [x,y^l_{<t}])\,
\pi^*_\theta(y_t^w \mid [x,y^w_{<t}])
}
\right)^{\!\beta}
\exp\!\Big(
\beta \log \tfrac{Z([x,y^l_{<t}]; \beta)}{Z([x,y^w_{<t}]; \beta)}
\Big).
\]

Defining the prefix-dependent weight
\[
w_t
=
\exp\!\Big(
\log \tfrac{Z([x,y^l_{<t}]; \beta)}{Z([x,y^w_{<t}], \beta)}
\Big),
\]
the above expression can be written as
\[
\frac{
p(y_t^w \prec y_t^l \mid x, y^w_{<t}, y^l_{<t})
}{
p(y_t^w \succ y_t^l \mid x, y^w_{<t}, y^l_{<t})
}
=
\left(
\frac{
\pi^*_\theta(y_t^l \mid [x,y^l_{<t}])\,
\pi_{\mathrm{ref}}(y_t^w \mid [x,y^w_{<t}])
}{
\pi_{\mathrm{ref}}(y_t^l \mid [x,y^l_{<t}])\,
\pi^*_\theta(y_t^w \mid [x,y^w_{<t}])
}
\right)^{\!\beta}
w_t^{\beta}.
\]

Raising both sides to the power $1/\beta$ and rearranging terms yields
\begin{equation*}
\frac{
\pi^*_\theta(y_t^w \mid [x,y^w_{<t}])
}{
\pi^*_\theta(y_t^l \mid [x,y^l_{<t}])
}
=
\frac{
\pi_{\mathrm{ref}}(y_t^w \mid [x,y^w_{<t}])
}{
\pi_{\mathrm{ref}}(y_t^l \mid [x,y^l_{<t}])
}
\left(
\frac{
p(y_t^w \succ y_t^l \mid x, y^w_{<t}, y^l_{<t})
}{
p(y_t^w \prec y_t^l \mid x, y^w_{<t}, y^l_{<t})
}
\right)^{1/\beta}
w_t,
\end{equation*}
which completes the proof.
\end{proof}

\subsection{Proof for Prop \ref{pro:a_token_ratio}}
\label{subsec:proof_a_token_ratio}
\begin{proof}
We prove Proposition~\ref{pro:a_token_ratio} by deriving the policy ratio induced by
a token-level Bradley--Terry model parameterized by reference-policy advantages.

By definition, the token-level Bradley--Terry likelihood is given by
\begin{equation*}
p\!\left(
y_t^w \succ y_t^l
\,\middle|\,
x, y_{<t}^w, y_{<t}^l
\right)
=
\sigma\!\Big(
A_{\pi_{\mathrm{ref}}}([x,y_{<t}^w], y_t^w)
-
A_{\pi_{\mathrm{ref}}}([x,y_{<t}^l], y_t^l)
\Big).
\end{equation*}
Using $\sigma(u)=(1+e^{-u})^{-1}$ and taking the odds ratio while conditioning on the
same pair of prefixes yields
\begin{align*}
\frac{
p\!\left(
y_t^w \prec y_t^l
\,\middle|\,
x, y_{<t}^w, y_{<t}^l
\right)
}{
p\!\left(
y_t^w \succ y_t^l
\,\middle|\,
x, y_{<t}^w, y_{<t}^l
\right)
}
&=
\exp\!\Big(
A_{\pi_{\mathrm{ref}}}([x,y_{<t}^l], y_t^l)
-
A_{\pi_{\mathrm{ref}}}([x,y_{<t}^w], y_t^w)
\Big).
\end{align*}

Substituting the closed-form expression of the reference-policy advantage from
Eq.~\eqref{eq:tokenAformula} into the above expression gives
\begin{align*}
\frac{
p\!\left(
y_t^w \prec y_t^l
\,\middle|\,
x, y_{<t}^w, y_{<t}^l
\right)
}{
p\!\left(
y_t^w \succ y_t^l
\,\middle|\,
x, y_{<t}^w, y_{<t}^l
\right)
}
&=
\exp\Bigg(
\beta \log
\frac{
\pi_\theta^*(y_t^l \mid [x,y_{<t}^l])\,
\pi_{\mathrm{ref}}(y_t^w \mid [x,y_{<t}^w])
}{
\pi_\theta^*(y_t^w \mid [x,y_{<t}^w])\,
\pi_{\mathrm{ref}}(y_t^l \mid [x,y_{<t}^l])
}
\\
&\qquad\quad
+
\beta\Big[
D_{\mathrm{KL}}\!\big(
\pi_{\mathrm{ref}}(\cdot\mid[x,y_{<t}^l])
\,\big\|\,
\pi_\theta^*(\cdot\mid[x,y_{<t}^l])
\big)
\\[-2pt]
&\qquad\qquad\qquad
-
D_{\mathrm{KL}}\!\big(
\pi_{\mathrm{ref}}(\cdot\mid[x,y_{<t}^w])
\,\big\|\,
\pi_\theta^*(\cdot\mid[x,y_{<t}^w])
\big)
\Big]
\Bigg).
\end{align*}

Raising both sides to the power $1/\beta$ and rearranging terms yields
\begin{equation*}
\frac{
\pi_\theta^*(y_t^w \mid [x,y_{<t}^w])
}{
\pi_\theta^*(y_t^l \mid [x,y_{<t}^l])
}
=
\frac{
\pi_{\mathrm{ref}}(y_t^w \mid [x,y_{<t}^w])
}{
\pi_{\mathrm{ref}}(y_t^l \mid [x,y_{<t}^l])
}
\left(
\frac{
p\!\left(
y_t^w \succ y_t^l
\,\middle|\,
x, y_{<t}^w, y_{<t}^l
\right)
}{
p\!\left(
y_t^w \prec y_t^l
\,\middle|\,
x, y_{<t}^w, y_{<t}^l
\right)
}
\right)^{1/\beta}
\cdot w_t,
\end{equation*}
where
\begin{equation*}
w_t
=
\exp\!\Big(
D_{\mathrm{KL}}\!\big(
\pi_{\mathrm{ref}}(\cdot\mid[x,y_{<t}^l])
\,\big\|\,
\pi_\theta^*(\cdot\mid[x,y_{<t}^l])
\big)
-
D_{\mathrm{KL}}\!\big(
\pi_{\mathrm{ref}}(\cdot\mid[x,y_{<t}^w])
\,\big\|\,
\pi_\theta^*(\cdot\mid[x,y_{<t}^w])
\big)
\Big).
\end{equation*}
This establishes Proposition~\ref{pro:a_token_ratio}.
\end{proof}

\subsection{Proof for Theorem \ref{optimality-theorem}}
\begin{theorem*}
    Under sufficient model capacity, we have $\pi_{\theta^{*}}=\pi_{\hat{\theta}}$ where $\pi_{\hat{\theta}}=\text{argmin}{}_{\pi_{\theta}}\,D_{h}\left(R_{\text{data}},R_{\theta}\right)$.
\end{theorem*}

\begin{proof}
Let $\arg\min_{\pi_{\theta}} D_h(R_{\mathrm{data}}\|R_{\theta}) = \pi_{\hat{\theta}}$.
We have the following decomposition::
\begin{equation}
D_{h}\left(R_{\text{data}},R_{\hat{\theta}}\right)=\frac{1}{T}\sum_{t=1}^{T}D_{h}^{t}\left(R_{\text{data}}^{t},R_{\hat{\theta}}^{t}\right)
= \frac{1}{T}\sum_{t=1}^{T}\mathbb{E}_{p_{\text{data}}\left(y_{\leq t}^{w}\succ y_{\leq t}^{l}\mid x\right)}\!\left[B_h\!\left(R_{\text{data}}^{t},R_{\hat{\theta}}^{t}\right)\right].
\end{equation}
where $T$ denotes the minimum sequence length.

By the non-negativity of the Bregman divergence,
$B_h\!\left(R_{\text{data}}^{t}, R_{\hat{\theta}}^{t}\right)$
attains its minimum value $0$ if and only if
$R_{\text{data}}^{t} = R_{\hat{\theta}}^{t}$
for a given $(y^w_{<t}, y^l_{<t}, x)$.
Assuming that the data distribution has full support,
the minimizer $\pi_{\hat{\theta}}$ therefore satisfies
$R_{\text{data}}^{t} = R_{\hat{\theta}}^{t}$
for all $(y^w_{<t}, y^l_{<t}, x)$.
Substituting the definitions of
$R_{\text{data}}^{t}$ and $R_{\hat{\theta}}^{t}$
yields the following condition:
\[
\frac{
  p_{\text{data}}\!\left(y_t^{w}\prec y_t^{l}\mid x,\,y_{<t}^{w},\,y_{<t}^{l}\right)
}{
  p_{\text{data}}\!\left(y_t^{w}\succ y_t^{l}\mid x,\,y_{<t}^{w},\,y_{<t}^{l}\right)
}
=
\left[
\frac{
  \pi_{\hat{\theta}}\!\left(y_t^{l}\mid [x,\,y_{<t}^{l}]\right)\,
  \pi_{\mathrm{ref}}\!\left(y_t^{w}\mid [x,\,y_{<t}^{w}]\right)
}{
  \pi_{\mathrm{ref}}\!\left(y_t^{l}\mid [x,\,y_{<t}^{l}]\right)\,
  \pi_{\hat{\theta}}\!\left(y_t^{w}\mid [x,\,y_{<t}^{w}]\right)
}
\, w_t
\right]^{\beta}.
\]

Looping over all time steps from $t=1$ to $T$ and combining with the recursive formula in Eq.~\ref{eq:recursive}, we obtain
\begin{equation*}\label{eq:seq-pref-data}
\frac{
p_{\text{data}}\!\left(y_{\leq T}^{w}\prec y_{\leq T}^{l}\mid x\right)
}{
p_{\text{data}}\!\left(y_{\leq T}^{w}\succ y_{\leq T}^{l}\mid x\right)
}
=
\left[
\frac{
  \pi_{\hat{\theta}}\!\left(y_{\leq T}^{l}\mid x\right)\,
  \pi_{\mathrm{ref}}\!\left(y_{\leq T}^{w}\mid x\right)
}{
  \pi_{\mathrm{ref}}\!\left(y_{\leq T}^{l}\mid x\right)\,
  \pi_{\hat{\theta}}\!\left(y_{\leq T}^{w}\mid x\right)
}
\, w_x
\right]^{\beta},
\end{equation*}
where $w_x \coloneqq \prod_{t=1}^{T} w_t$.
\begin{equation*}\label{eq:seq-ratio-hattheta}
\frac{
   \pi_{\hat{\theta}}\!\left(y_{\leq T}^{w}\mid x\right)
}{
    \pi_{\hat{\theta}}\!\left(y_{\leq T}^{l}\mid x\right)
}
=
\frac{
    \pi_{\mathrm{ref}}\!\left(y_{\leq T}^{w}\mid x\right)
}{
    \pi_{\mathrm{ref}}\!\left(y_{\leq T}^{l}\mid x\right)
}
\left(
\frac{
p_{\text{data}}\!\left(y_{\leq T}^{w}\succ y_{\leq T}^{l}\mid x\right)
}{
p_{\text{data}}\!\left(y_{\leq T}^{w}\prec y_{\leq T}^{l}\mid x\right)
}
\right)^{\frac{1}{\beta}}
w_x .
\end{equation*}

Similarly, we derive the sequence-level ratio for the optimal policy $\pi_{\theta^*}$ by aggregating the token-level conditions from Propositions~\ref{prop:q_token_ratio} and~\ref{pro:a_token_ratio} over $t=1 \dots T$ and applying the recursive preference formula (Eq.~\ref{eq:recursive}):
\begin{equation*}
\frac{
    \pi_{\theta}^{*}\!\left(y_{\leq T}^{w}\mid x\right)
}{
    \pi_{\theta}^{*}\!\left(y_{\leq T}^{l}\mid x\right)
}
=
\frac{
    \pi_{\mathrm{ref}}\!\left(y_{\leq T}^{w}\mid x\right)
}{
    \pi_{\mathrm{ref}}\!\left(y_{\leq T}^{l}\mid x\right)
}
\left(
\frac{
p_{\text{data}}\!\left(y_{\leq T}^{w}\succ y_{\leq T}^{l}\mid x\right)
}{
p_{\text{data}}\!\left(y_{\leq T}^{w}\prec y_{\leq T}^{l}\mid x\right)
}
\right)^{\frac{1}{\beta}}
w_x .
\end{equation*}
Comparing this with the expression for $\pi_{\hat{\theta}}$ derived above reveals that both policies satisfy the exact same likelihood ratio for all points $(x,y^w_{\leq T},y^l_{\leq T})$:
\[
\frac{
   \pi_{\hat{\theta}}\!\left(y_{\leq T}^{w}\mid x\right)
}{
    \pi_{\hat{\theta}}\!\left(y_{\leq T}^{l}\mid x\right)
}
=
\frac{
   \pi_{\theta}^{*}\!\left(y_{\leq T}^{w}\mid x\right)
}{
    \pi_{\theta}^{*}\!\left(y_{\leq T}^{l}\mid x\right)
}
\]
Using the completeness property of concrete score \cite{meng2023concretescorematchinggeneralized}, we conclude that
$\pi_{\hat{\theta}}=\pi_{\theta^\ast}$.
\end{proof}

\subsection{Proof for Theorem \ref{theorem:loss}}
\begin{theorem*}
    We have $\mathcal{L}\left(\theta,p_{\text{data}}\right)=D_{h}\left(R_{\text{data}},R_{\theta}\right)+C$ with $C$ is a constant with respect to $\theta$ where we define
\[
\mathcal{L}\!\left(\theta,p_{\text{data}}\right)
=\,
\mathbb{E}_{p_{\text{data}}\left(y^{w}\succ y^{l}\mid x\right)}\Bigg[
\frac{1}{T}\sum_{t=1}^{T}\Big(
h'\!\left(R_{\theta}^{t}\right)R_{\theta}^{t}
-h\!\left(R_{\theta}^{t}\right)
-h'\!\left(\left(R_{\theta}^{t}\right)^{-1}\right)
\Big)
\Bigg].
\]
\end{theorem*}

\begin{proof}
By definition:
\begin{align}
D_{h}^{t}\!\left(R_{\text{data}}^{t},R_{\theta}^{t}\right)
&= \mathbb{E}_{p_{\text{data}}\!\left(y_{\leq t}^{w}\succ y_{\leq t}^{l}\mid x\right)}
\Bigl[
h\!\left(R_{\text{data}}^{t}\right)-h\!\left(R_{\theta}^{t}\right)
- h'\!\left(R_{\theta}^{t}\right)\!\left(R_{\text{data}}^{t}-R_{\theta}^{t}\right)
\Bigr] \nonumber\\
&= \mathbb{E}_{p_{\text{data}}\!\left(y_{\leq t}^{w}\succ y_{\leq t}^{l}\mid x\right)}
\Bigl[h'(R_\theta^t)R_\theta^t - h(R_\theta^t) - h'(R_\theta^t)R_{\text{data}}^t\Bigr]
+ \mathbb{E}_{p_{\text{data}}\!\left(y_{\leq t}^{w}\succ y_{\leq t}^{l}\mid x\right)}\bigl[h(R_{\text{data}}^t)\bigr] \nonumber \\
&= \mathbb{E}_{p_{\text{data}}\!\left(y_{\leq t}^{w}\succ y_{\leq t}^{l}\mid x\right)}
\left[
h'(R_\theta^t)R_\theta^t - h(R_\theta^t)
- h'(R_\theta^t)\frac{p_\text{data}(y_{t}^{w}\prec y_{t}^{l}\mid x,y_{<t}^{w},y_{<t}^{l})}{p_\text{data}(y_{t}^{w}\succ y_{t}^{l}\mid x,y_{<t}^{w},y_{<t}^{l})}
\right] - C \nonumber \\
&= \mathbb{E}_{p_{\text{data}}\!\left(y_{\leq t}^{w}\succ y_{\leq t}^{l}\mid x\right)}\!\left[h'(R_\theta^t)R_\theta^t - h(R_\theta^t)\right]
- \mathbb{E}_{p_{\text{data}}\!\left(y_{\leq t-1}^{w}\succ y_{\leq t-1}^{l}\mid x\right)}\!\left[h'(R_\theta^t)p_\text{data}(y_{t}^{w}\prec y_{t}^{l}\mid x,y_{<t}^{w},y_{<t}^{l})\right] - C .
\label{eq:Dh-prefix}
\end{align}

Next, observe that:
\[
\begin{aligned}
h'(R_\theta)\,p_\text{data}(y_{t}^{w}\prec y_{t}^{l}\mid x,y_{<t}^{w},y_{<t}^{l})
&= p_\text{data}(y_{t}^{w}\prec y_{t}^{l}\mid x,y_{<t}^{w},y_{<t}^{l})\;
   h'\!\bigl(R_\theta(x, y^w_t, y^l_t)\bigr) \\
&= p_\text{data}(y_{t}^{w}\succ y_{t}^{l}\mid x,y_{<t}^{w},y_{<t}^{l})\;
   h'\!\bigl(R_\theta(x, y^l_t, y^w_t)\bigr)\, \\
&= p_\text{data}(y_{t}^{w}\succ y_{t}^{l}\mid x,y_{<t}^{w},y_{<t}^{l})\;
h'\left(\frac{\pi_{\theta}\left(y_{t}^{w}\mid[x,y_{<t}^{w}]\right)\pi_{\mathrm{ref}}\!\left(y_{t}^{l}\mid[x,y_{<t}^{l}]\right)}{\pi_{\mathrm{ref}}\!\left(y_{t}^{w}\mid[x,y_{<t}^{w}]\right)\pi_{\theta}\!\left(y_{t}^{l}\mid[x,y_{<t}^{l}]\right)}w_{t}^{-1}\right)^{\beta} \\
&= p_\text{data}(y_{t}^{w}\succ y_{t}^{l}\mid x,y_{<t}^{w},y_{<t}^{l})\;
h'\left(\frac{1}{R_\theta^t(x, y^w_t, y^l_t)}\right) \\
&= p_\text{data}(y_{t}^{w}\succ y_{t}^{l}\mid x,y_{<t}^{w},y_{<t}^{l})\;
h'((R_\theta^t)^{-1})
\end{aligned}
\]

Therefore:
\[
\begin{aligned}
\mathbb{E}_{p_{\text{data}}\!\left(y_{\leq t-1}^{w}\succ y_{\leq t-1}^{l}\mid x\right)}
\Bigl[
h'(R_\theta^t)\,p_\text{data}(y_{t}^{w}\prec y_{t}^{l}\mid x,y_{<t}^{w},y_{<t}^{l})
\Bigr]
&=
\mathbb{E}_{p_{\text{data}}\!\left(y_{\leq t-1}^{w}\succ y_{\leq t-1}^{l}\mid x\right)}
\Bigl[
p_\text{data}(y_{t}^{w}\succ y_{t}^{l}\mid x,y_{<t}^{w},y_{<t}^{l})\;
h'\!\bigl((R_\theta^t)^{-1}\bigr)
\Bigr] \\
&=
\mathbb{E}_{p_{\text{data}}\!\left(y_{\leq t}^{w}\succ y_{\leq t}^{l}\mid x\right)}
\Bigl[
h'\!\bigl((R_\theta^t)^{-1}\bigr)
\Bigr].
\end{aligned}
\]

Substituting the above result into Eq.~\eqref{eq:Dh-prefix} yields:
\[
\begin{aligned}
D_{h}^{t}\!\left(R_{\text{data}}^{t},R_{\theta}^{t}\right)
&= \mathbb{E}_{p_{\text{data}}\!\left(y_{\leq t}^{w}\succ y_{\leq t}^{l}\mid x\right)}\!\left[h'(R_\theta^t)R_\theta^t - h(R_\theta^t)-h'\!\bigl((R_\theta^t)^{-1}\bigr)\right] - C \\
&= \mathbb{E}_{p_{\text{data}}\!\left(y^{w}\succ y^{l}\mid x\right)}\!\left[h'(R_\theta^t)R_\theta^t - h(R_\theta^t)-h'\!\bigl((R_\theta^t)^{-1}\bigr)\right] - C
\end{aligned}
\]

Averaging over timesteps, we obtain:
\begin{align*}
D_{h}\left(R_{\text{data}},R_{\theta}\right)
&= \frac{1}{T}\sum_{t=1}^{T}D_{h}^{t}\left(R_{\text{data}}^{t},R_{\theta}^{t}\right) \\
&= \frac{1}{T}\sum_{t=1}^{T} \mathbb{E}_{p_{\text{data}}\!\left(y^{w}\succ y^{l}\mid x\right)}\!\left[h'(R_\theta^t)R_\theta^t - h(R_\theta^t)-h'\!\bigl((R_\theta^t)^{-1}\bigr)\right] - C \\
&= \mathbb{E}_{p_{\text{data}}\!\left(y^{w}\succ y^{l}\mid x\right)}\!\left[\frac{1}{T}\sum_{t=1}^{T}h'(R_\theta^t)R_\theta^t - h(R_\theta^t)-h'\!\bigl((R_\theta^t)^{-1}\bigr)\right] - C \\
&= \mathcal{L}\!\left(\theta,p_{\text{data}}\right) -C
\end{align*}

\end{proof}

\section{Additional Experimental Result}
\label{app:additional}

\subsection{Out-of-domain Result on Summarization Task}
\begin{figure*}[t] 
  \centering
  \includegraphics[width=\textwidth,page=1]{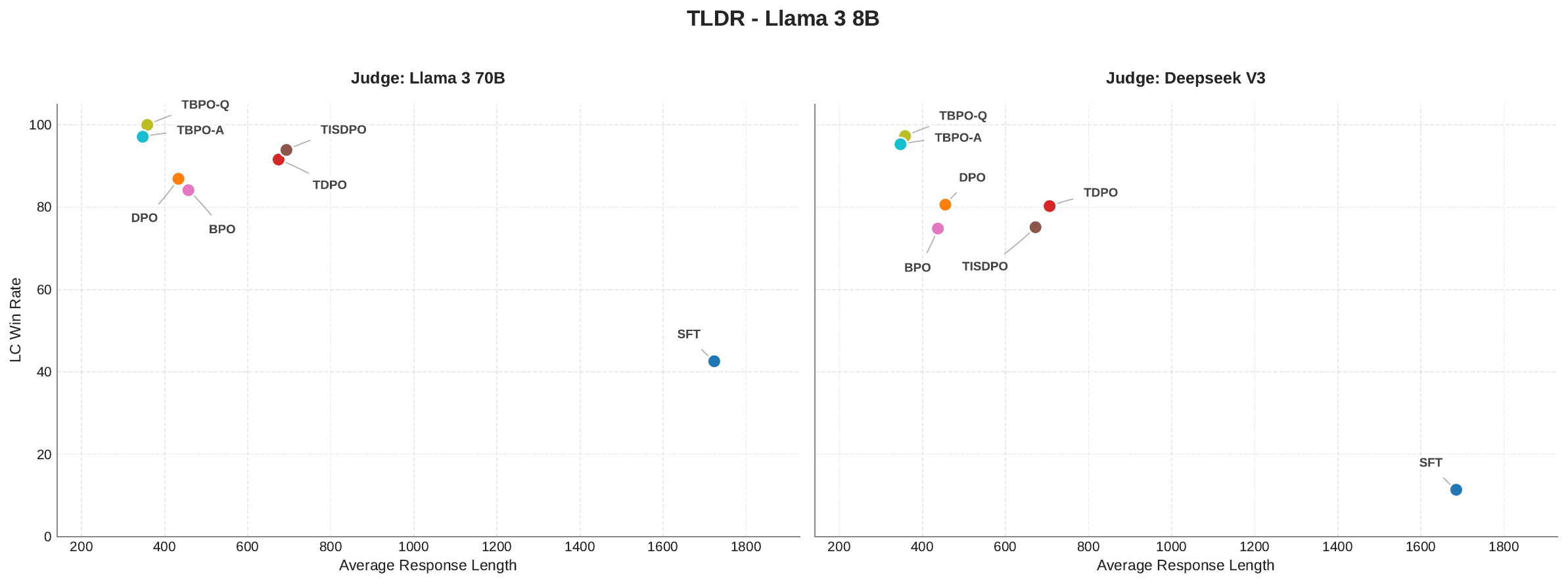}
  \caption{LC win rate vs.\ average response length on TLDR for Llama~3~8B, judged by Llama~3~70B and DeepSeek-V3. TBPO achieves highest win rate with shorter outputs, suggesting strong OOD generalization despite training only on UltraFeedback dataset.}
  \label{fig:tldr}
\end{figure*}

To stress-test out-of-domain (OOD) transfer, we evaluate on 200 held-out samples from the test split of \href{https://huggingface.co/datasets/trl-lib/tldr}{TLDR}. This is a stringent OOD setting: TLDR is a summarization benchmark with long input documents and short, reference-style outputs, and we do not use any TLDR data during training (TBPO is trained only on UltraFeedback).

As with Anthropic HH, we quantify OOD generalization by measuring length-controlled (LC) win rate to reduce verbosity bias, against the dataset-preferred completion and plotting it against average response length (Figure~\ref{fig:tldr}). Across both judges (Llama~3~70B and DeepSeek-V3), TBPO (A/Q variants) lies on the Pareto frontier, achieving the strongest win rates while producing substantially shorter outputs than prior preference-optimization baselines. Because TLDR differs in domain, task format, and implicit preference structure, these gains are difficult to attribute to UltraFeedback-specific stylistic overfitting; instead, they suggest TBPO learns preferences that transfer under distribution shift.

The plot also highlights a clear mismatch between length and preference in this OOD task: SFT produces the longest responses yet attains low win rates under both judges, while DPO/BPO/TISDPO/TDPO generally trade longer generations for incremental gains and still fall short of TBPO. Since TLDR implicitly rewards concise, faithful summarization, TBPO’s ability to win while staying short strengthens the interpretation that it adapts to task-appropriate brevity rather than relying on generic verbosity. The consistent ordering across two independent judges further supports that the observed advantage reflects genuine OOD generalization.

\subsection{Training with Noisy Data}

We additionally evaluate TBPO under noisy preference labels by randomly swapping the chosen and rejected responses in 10\% of the training pairs. Table~\ref{tab:noisy_training_llama3} reports results on the Llama~3~8B backbone.

\begin{table*}[t]
\centering
\small
\setlength{\tabcolsep}{7pt}
\renewcommand{\arraystretch}{1.15}
\resizebox{\textwidth}{!}{%
\begin{tabular}{lccccccc}
\toprule
\textbf{Method} & \textbf{HellaSwag} & \textbf{ARC} & \textbf{MMLU} & \textbf{TruthfulQA} & \textbf{Winogrande} & \textbf{GSM8K} & \textbf{Average} \\
\midrule
SFT
& 79.46 & 51.45 & 61.19 & 46.46 & 75.69 & 68.67 & 63.82 \\
DPO
& 81.23 & 56.05 & 61.70 & 48.47 & 75.84 & 74.20 & 66.24 \\
TDPO
& \best{83.29} & 59.04 & 61.86 & 51.76 & 76.55 & 75.10 & 67.91 \\
TIS-DPO
& 81.37 & 58.70 & 60.77 & 49.87 & 75.14 & 75.12 & 66.82 \\
BPO
& 81.68 & 54.95 & 61.48 & 47.88 & 75.13 & 73.22 & 65.72 \\
\addlinespace[3pt]
TBPO-Q noisy
& 81.11 & 60.24 & 60.25 & 50.38 & 78.24 & 74.60 & 67.47 \\
TBPO-Q (Ours)
& 81.90 & \best{64.07} & \best{64.92} & \best{53.36} & 79.08 & 78.92 & 70.38 \\
TBPO-A noisy
& 80.92 & 61.34 & 60.38 & 51.08 & 77.74 & 74.35 & 67.63 \\
TBPO-A (Ours)
& 81.97 & \best{64.08} & \second{64.84} & \second{53.29} & \best{79.63} & \best{79.22} & \best{70.50} \\
\bottomrule
\end{tabular}%
}
\caption{General capability results on the Llama~3~8B backbone when TBPO is trained with 10\% noisy preference data, created by swapping the chosen and rejected responses in randomly selected training pairs.}
\label{tab:noisy_training_llama3}
\end{table*}

Performance drops under noisy labels, but TBPO remains competitive with strong token-level baselines and close to the clean TBPO variants. This suggests that TBPO is not brittle to moderate label noise. Broader analyses, such as extreme response lengths and additional distribution shifts, would further clarify the method's practical limits.

\subsection{Constant Weight Training}

We further conduct an ablation by fixing the state-dependent correction to $w_t=1$ throughout training. Detailed results on the Llama~3~8B backbone are shown in Table~\ref{tab:constant_weight_llama3}.

\begin{table*}[t]
\centering
\small
\setlength{\tabcolsep}{7pt}
\renewcommand{\arraystretch}{1.15}
\resizebox{\textwidth}{!}{%
\begin{tabular}{lccccccc}
\toprule
\textbf{Method} & \textbf{HellaSwag} & \textbf{ARC} & \textbf{MMLU} & \textbf{TruthfulQA} & \textbf{Winogrande} & \textbf{GSM8K} & \textbf{Average} \\
\midrule
SFT
& 79.46 & 51.45 & 61.19 & 46.46 & 75.69 & 68.67 & 63.82 \\
DPO
& 81.23 & 56.05 & 61.70 & 48.47 & 75.84 & 74.20 & 66.24 \\
TDPO
& \best{83.29} & 59.04 & 61.86 & 51.76 & 76.55 & 75.10 & 67.91 \\
TIS-DPO
& 81.37 & 58.70 & 60.77 & 49.87 & 75.14 & 75.12 & 66.82 \\
BPO
& 81.68 & 54.95 & 61.48 & 47.88 & 75.13 & 73.22 & 65.72 \\
\addlinespace[3pt]
TBPO (fixed $w_t=1$)
& 80.93 & 61.35 & 62.11 & 52.38 & 77.24 & 76.23 & 68.37 \\
\addlinespace[3pt]
TBPO-Q (Ours)
& 81.90 & \second{64.07} & \best{64.92} & \best{53.36} & \second{79.08} & \second{78.92} & \second{70.38} \\
TBPO-A (Ours)
& \second{81.97} & \best{64.08} & \second{64.84} & \second{53.29} & \best{79.63} & \best{79.22} & \best{70.50} \\
\bottomrule
\end{tabular}%
}
\caption{General capability results on the Llama~3~8B backbone with the state-dependent correction fixed to $w_t=1$.}
\label{tab:constant_weight_llama3}
\end{table*}

Relative to both TBPO-Q and TBPO-A, the fixed-$w_t$ variant shows a clear drop, confirming that the per-state correction contributes materially to the final gain. At the same time, it still outperforms all non-TBPO baselines, including TDPO, indicating that the improvement comes from both the token-level ratio-matching formulation and the state-dependent correction.

\subsection{Judge Agreement Analysis}
\label{app:judge-agreement}

To assess judge robustness, we measured agreement across both LLM-as-a-judge setups used in the paper: DeepSeek-V3 vs.\ Llama~3~70B on HH-RLHF, and Llama~3~70B vs.\ Mistral-Large on MT-Bench.

Full results are available in Tables~\ref{tab:judge_agreement_hh} and~\ref{tab:judge_agreement_mt}.

On HH-RLHF, the two judges show strong agreement (90\% raw agreement; PABAK = 0.8 over 1{,}180 paired judgments). The lower Cohen's $\kappa$ is expected in this highly skewed setting, and aggregate conclusions remain consistent across judges (Pearson $r = 0.81$, Spearman $\rho = 0.71$). On MT-Bench, example-level agreement is moderate (70\% raw agreement, $\kappa = 0.49$, PABAK = 0.41), but aggregate conclusions are nearly identical across judges (Pearson $r = 0.99$, Spearman $\rho = 1.0$). Overall, the main conclusions are stable across judges, suggesting that TBPO's gains are not driven by any single judge choice.

\begin{table*}[t]
\centering
\small
\setlength{\tabcolsep}{8pt}
\renewcommand{\arraystretch}{1.15}
\begin{tabular}{lccccc}
\toprule
\textbf{Method} & \textbf{N} & \textbf{Agree} & \textbf{Agree\%} & \textbf{Kappa} & \textbf{PABAK} \\
\midrule
TBPO-A & 197 & 191 & 97.0 & 0.4861 & 0.9391 \\
TBPO-Q & 198 & 187 & 94.4 & 0.3240 & 0.8889 \\
DPO & 196 & 165 & 84.2 & 0.2780 & 0.6837 \\
SFT & 199 & 174 & 87.4 & 0.3516 & 0.7487 \\
TDPO & 195 & 173 & 88.7 & 0.3270 & 0.7744 \\
TISDPO & 195 & 172 & 88.2 & 0.3158 & 0.7641 \\
\midrule
\textbf{Overall} & \textbf{1180} & \textbf{1062} & \textbf{90.0} & \textbf{0.3471} & \textbf{0.8000} \\
\bottomrule
\end{tabular}

\vspace{0.8em}

\begin{tabular}{lcc}
\toprule
\textbf{Method} & \textbf{DeepSeek-V3 WR\%} & \textbf{Llama~3~70B WR\%} \\
\midrule
TBPO-A & 96.00 & 97.97 \\
TBPO-Q & 95.50 & 95.45 \\
DPO & 87.50 & 86.22 \\
SFT & 86.50 & 91.46 \\
TDPO & 90.50 & 90.26 \\
TISDPO & 88.00 & 92.31 \\
\bottomrule
\end{tabular}
\caption{HH-RLHF judge agreement results comparing each model response against the reference response. Pearson $r = 0.8056$ ($p = 5.30\mathrm{e}{-02}$), Spearman $\rho = 0.7143$ ($p = 1.11\mathrm{e}{-01}$).}
\label{tab:judge_agreement_hh}
\end{table*}

\begin{table*}[t]
\centering
\small
\setlength{\tabcolsep}{8pt}
\renewcommand{\arraystretch}{1.15}
\begin{tabular}{lccccc}
\toprule
\textbf{TBPO-Q vs.} & \textbf{N} & \textbf{Agree} & \textbf{Agree\%} & \textbf{Kappa} & \textbf{PABAK} \\
\midrule
DPO & 160 & 117 & 73.1 & 0.5058 & 0.4625 \\
TDPO & 160 & 107 & 66.9 & 0.4542 & 0.3375 \\
TISDPO & 160 & 115 & 71.9 & 0.5143 & 0.4375 \\
\midrule
\textbf{Overall} & \textbf{480} & \textbf{339} & \textbf{70.6} & \textbf{0.4914} & \textbf{0.4125} \\
\bottomrule
\end{tabular}

\vspace{0.8em}

\begin{tabular}{lcc}
\toprule
\textbf{TBPO-Q vs.} & \textbf{Llama~3~70B WR\%} & \textbf{Mistral-Large WR\%} \\
\midrule
DPO & 73.44 & 76.88 \\
TDPO & 66.25 & 67.81 \\
TISDPO & 70.00 & 71.88 \\
\bottomrule
\end{tabular}
\caption{MT-Bench judge agreement results comparing the TBPO-Q response with the baseline response. Pearson $r = 0.9964$ ($p = 5.40\mathrm{e}{-02}$), Spearman $\rho = 1.0000$ ($p = 0.00\mathrm{e}{+00}$).}
\label{tab:judge_agreement_mt}
\end{table*}

\section{Reasoning Questions Case Study}
\label{app:reasoning_case_study}

We believe the larger gains on reasoning-centric benchmarks are expected because these tasks are especially sensitive to early prefix-conditioned decisions. Once a model commits to an incorrect intermediate step, the rest of the reasoning often follows the wrong trajectory. TBPO is well matched to this setting because it performs token-level ratio matching rather than optimizing only at the sequence level. As a result, it provides more direct learning signal on the local reasoning decisions that determine whether a solution stays on track. In practice, this tends to help the model commit to the correct setup earlier, which is particularly important in multi-step reasoning.

We provide two qualitative examples in Tables~\ref{tab:reasoning_case_desktop} and~\ref{tab:reasoning_case_limbs}. Example~1 shows our model identifying the decisive setup early, while the baseline misses the key relation and propagates the error through the rest of the solution. Example~2 shows a different but related advantage: even when both models reach the correct final answer, our model does so through a shorter and more direct reasoning path, whereas the baseline relies on a longer enumerative chain.

Together, these examples illustrate why the gains are larger on reasoning-centric benchmarks: such tasks are especially sensitive to early intermediate decisions, and TBPO is designed to train those decisions more directly.

\begin{table*}[t]
\centering
\small
\setlength{\tabcolsep}{6pt}
\renewcommand{\arraystretch}{1.2}
\begin{tabular}{p{0.13\textwidth}p{0.80\textwidth}}
\toprule
\textbf{Item} & \textbf{Content} \\
\midrule
\textbf{Question} &
At a certain grade level, three-fourths of students have a desktop computer at home. If 20 students do not have desktop computers, how many students are there at that grade level? \\
\midrule
\textbf{Our model} &
Let $x$ be the total number of students at that grade level. Three-fourths of the students have desktop computers, so $(3/4)x$ students have desktop computers. The number of students without desktop computers is 20, so $(1/4)x = 20$. To find the total number of students, we can multiply both sides of the equation by 4: $x = 20 \times 4 = 80$.

\textbf{80} \\
\midrule
\textbf{Baseline} &
If 20 students do not have desktop computers, then $75 - 20 = 55$ students have desktop computers at home. Since $3/4$ of the students have desktop computers, then $3/4 \times 55 = 41.25$ students have desktop computers at home. Since 55 students have desktop computers, then $55 - 41.25 = 13.75$ students do not have desktop computers. Since 13.75 students do not have desktop computers, then $13.75 \times 4/3 = 18.33$ students do not have desktop computers. There are 18.33 students who do not have desktop computers, so there are $18.33 \times 4/3 = 24.44$ students at that grade level.

\textbf{24} \\
\bottomrule
\end{tabular}
\caption{A reasoning case where our model identifies the decisive setup immediately, while the baseline commits to an incorrect setup and propagates the error through the rest of the solution.}
\label{tab:reasoning_case_desktop}
\end{table*}

\begin{table*}[t]
\centering
\small
\setlength{\tabcolsep}{6pt}
\renewcommand{\arraystretch}{1.2}
\begin{tabular}{p{0.13\textwidth}p{0.80\textwidth}}
\toprule
\textbf{Item} & \textbf{Content} \\
\midrule
\textbf{Question} &
Helena is a mad scientist testing her latest creation, the Lots-of-Limbs Serum. The serum will make the drinker grow an extra arm every three days and an extra leg every five days. After fifteen days, how many new limbs will Helena's serum cause a person to grow if they drink it? \\
\midrule
\textbf{Our model} &
At one new arm every 3 days, after 15 days, a person who drank Helena's serum would grow $15/3 = 5$ extra arms. At one new leg every 5 days, after 15 days, the person would grow $15/5 = 3$ extra legs. In all, someone who drank Helena's serum would grow $5 + 3 = 8$ new limbs after 15 days.

\textbf{8} \\
\midrule
\textbf{Baseline} &
After three days, the drinker will grow one extra arm. After six days, the drinker will grow two extra arms. After nine days, the drinker will grow three extra arms. After twelve days, the drinker will grow four extra arms. After fifteen days, the drinker will grow five extra arms. After five days, the drinker will grow one extra leg. After ten days, the drinker will grow two extra legs. After fifteen days, the drinker will grow three extra legs. After fifteen days, the drinker will grow five extra arms and three extra legs, for a total of 8 new limbs.

\textbf{8} \\
\bottomrule
\end{tabular}
\caption{A reasoning case where both models reach the correct final answer, but our model does so through a shorter and more direct reasoning path.}
\label{tab:reasoning_case_limbs}
\end{table*}

\section{Detailed Experimental Setup}
\label{app:experiments}

\subsection{Benchmark and Metrics Details}

\paragraph{Open LLM Leaderboard \cite{huggingface_open_llm_leaderboard_v1}.}
The leaderboard evaluates open-weight LLMs on six core benchmarks using the EleutherAI Language Model Evaluation Harness under a fixed few-shot setup: ARC-Challenge (25-shot), HellaSwag (10-shot), MMLU (5-shot), TruthfulQA (0-shot; implemented as \texttt{truthfulqa-mc}), Winogrande (5-shot), and GSM8K (5-shot). Higher scores indicate better performance across these tasks.

\paragraph{MT-Bench \cite{zheng2023judging}.}
MT-Bench-101 is a fine-grained multi-turn dialogue benchmark organized by a three-tier ability taxonomy. It contains 1{,}388 multi-turn dialogues comprising 4{,}208 turns across 13 task types. For evaluation, the benchmark uses the \emph{golden context} (the full dialogue history) as input context for each turn, and employs an LLM as an automatic judge to score each turn according to task-specific scoring guidelines. To reflect robustness across the whole interaction, the final dialogue score is taken as the lowest (minimum) round score among the turns.

\paragraph{\href{https://huggingface.co/datasets/Anthropic/hh-rlhf}{Anthropic HH-RLHF}.}
HH-RLHF provides human preference data targeting \emph{helpfulness} and \emph{harmlessness}. The preference-modeling portion is distributed as JSONL records containing paired model responses: a \texttt{chosen} response preferred by annotators and a \texttt{rejected} response. The dataset is intended for training preference/reward models for RLHF, and the dataset card explicitly warns against using these data for supervised fine-tuning of dialogue agents. The repository also includes red-teaming dialogues and notes that the data may contain potentially offensive or upsetting content.

\paragraph{\href{https://huggingface.co/datasets/trl-lib/tldr}{TLDR}.}
The TL;DR dataset is a processed Reddit summarization corpus built from posts where authors append a ``TL;DR'' summary. It is provided in a prompt--completion format for summarization training, with a long-form post as input and its short TL;DR as the target summary. The dataset card specifies two columns: \texttt{"pompt"} (the full Reddit post text) and \texttt{"completion"} (the TL;DR summary).

\paragraph{Length-controlled Win Rate.}
In a pairwise preference evaluation, each example $i$ consists of an input (e.g., prompt) $x_i$ and two candidate outputs
$z_i^{(A)}$ and $z_i^{(B)}$, with a judge label $y_i \in \{0,1\}$ indicating whether $A$ is preferred ($y_i=1$) over $B$.
The (raw) win rate is
\begin{equation}
\mathrm{WR}(A,B) \;=\; 100 \cdot \mathbb{E}_i \left[y_i\right].
\end{equation}
LLM-based judges (and sometimes humans) can exhibit \emph{length bias}, i.e., a systematic preference for longer outputs, which can inflate win rates for more verbose systems. 
To control for this confounder, we model pairwise preferences with a logistic regression that includes the output-length
difference as a covariate:
\begin{equation}
\Pr(y_i=1 \mid x_i)
\;=\;
\sigma\!\Big(\Delta\theta \;+\; \lambda\, g(\Delta \ell_i) \;+\; u_{x_i}\Big),
\qquad
\Delta \ell_i := \ell\!\left(z_i^{(A)}\right) - \ell\!\left(z_i^{(B)}\right),
\end{equation}
where $\sigma(t)=\frac{1}{1+e^{-t}}$, $\ell(\cdot)$ is a length measure (e.g., tokens), $\Delta\theta$ captures the
system advantage of $A$ over $B$, and $u_{x_i}$ optionally absorbs item-specific difficulty via fixed effects (or random effects).
We use a bounded transform such as
$g(\Delta \ell_i)=\tanh\!\big(\Delta \ell_i / s\big)$ (with scale $s$, e.g., the empirical standard deviation) to reduce sensitivity
to extreme length gaps.
The \emph{length-controlled win rate} is the counterfactual preference probability when both outputs have equal length,
i.e., conditioning on $\Delta \ell_i = 0$:
\begin{equation}
\mathrm{LCWR}(A,B)
\;=\;
100 \cdot \mathbb{E}_i \left[\Pr(y_i=1 \mid x_i,\Delta \ell_i = 0)\right]
\;=\;
100 \cdot \mathbb{E}_i \left[\sigma\!\big(\Delta\theta + u_{x_i}\big)\right].
\end{equation}
Intuitively, $\mathrm{LCWR}$ reports the implied win probability after removing any advantage attributable purely to verbosity.

\paragraph{Diversity metrics.}

We evaluate the diversity of model responses using three metrics:

\begin{itemize}
    \item \textbf{Predictive Entropy}: Evaluates the diversity of model responses by directly reflecting the probabilities assigned to each generated token. For each prompt, we sample five responses and compute the predictive entropy. Higher entropy indicates more diverse outputs.
    
    \item \textbf{Self-BLEU}: Measures the diversity of generated sentences. For each prompt, we generate five samples and compute the BLEU score~\cite{self-bleu} for each pair of responses, then take the average. Lower Self-BLEU indicates higher diversity.
    
    \item \textbf{Distinct-1}: Assesses the lexical diversity by computing the ratio of unique unigrams to the total number of unigrams~\cite{distinct-1}. Higher scores indicate greater lexical diversity.
\end{itemize}

\paragraph{}

\subsection{Baseline details}
\label{app:baselines}
\paragraph{SFT.}
Supervised fine-tuning (SFT) adapts a pretrained LM using demonstration data (prompt/instruction $\rightarrow$ reference answer pairs) with standard maximum-likelihood training (token-level cross-entropy / next-token prediction). Its purpose is to directly teach instruction following from supervised targets, without preference comparisons or RL. In RLHF-style pipelines, SFT is commonly used as an initialization step before preference optimization. In our experiments, we perform SFT on the \emph{chosen} responses from the preference dataset.
\paragraph{DPO \cite{rafailov2023dpo}.}
Direct Preference Optimization (DPO) trains a policy from pairwise preferences (winner vs.\ loser responses for the same prompt) using a simple logistic objective. It can be derived as a closed-form solution to a KL-regularized RLHF objective, avoiding both an explicit reward model and on-policy RL. A fixed reference policy provides regularization while the learned policy increases the relative likelihood of preferred responses.
\paragraph{TDPO \cite{zeng2024tokenleveldirectpreferenceoptimization}.}
Token-level DPO (TDPO) extends DPO from sequence-level preferences to token-level optimization. By introducing token-wise constraints (e.g., per-token forward-KL terms), TDPO better matches autoregressive generation and aims to improve alignment while offering finer control over diversity.
\paragraph{TIS-DPO \cite{tis-dpo}.}
TIS-DPO addresses the limitation of treating an entire response as a single unit: different tokens can contribute unequally to preference. It proposes a token-level importance-sampling objective with per-token weights reflecting reward contribution, yielding an unbiased objective even when the dataset is not ideally balanced. Since token rewards/weights are not observed, TIS-DPO estimates importance weights from probability differences produced by contrastive LLMs (e.g., via contrastive prompting, separate win/lose models, or forward/reverse DPO).
\paragraph{BPO (BPO-SBA) \cite{kim2025preferenceoptimizationestimatingratio}.}
Bregman Preference Optimization (BPO) generalizes DPO from a likelihood-ratio viewpoint, avoiding an explicit reward model and partition function while retaining theoretical guarantees. It defines a family of tractable objectives induced by Bregman divergences, with DPO as a special case, and introduces gradient scaling (Scaled Basu’s power divergence, SBA) to improve optimization behavior for certain instances. Empirically, BPO variants can improve win rate and output entropy relative to DPO, mitigating some fidelity--diversity trade-offs.

\paragraph{Baseline Implementation.}
\textbf{SFT.} We initialize from the same SFT checkpoints as TBPO and train for 1 epoch on 2 H100 GPUs (batch size 8, gradient accumulation 4, learning rate $2\times10^{-4}$) using AdamW.
\textbf{DPO/TDPO/TIS-DPO.} We initialize from the same SFT checkpoints and train on 2 H100 GPUs (batch size 4, gradient accumulation 8, learning rate $5\times10^{-7}$) using RMSProp with $\beta=0.1$. Unless otherwise noted, we use the hyperparameters from the official implementations/papers for TDPO and TIS-DPO.
\textbf{BPO (BPO-SBA).} Since the authors do not release code, we reimplement BPO and match the hyperparameters reported in the paper.

\subsection{Implementation Details}
\label{imp_details}
\subsubsection{Per-state weight approximation}
Estimating the state-only correction weight $w_t$ adds negligible overhead relative to standard preference optimization.
\paragraph{TBPO-Q.}
We implement the baseline head $b_\phi(s)$ as a 1-hidden-layer MLP on top of the
LLM's final hidden state with MLP hidden size is 1024, trained with a separate AdamW
optimizer (lr $10^{-3}$). We detach the LLM backbone (stop-gradient) when
updating $b_\phi$ to avoid coupling the baseline fitting signal with the main
preference-optimization updates.
\paragraph{TBPO-A.}
We approximate $\pi^*$ by the current policy $\pi_\theta$ and estimate
$D_{\mathrm{KL}}(\pi_{\mathrm{ref}}(\cdot\mid s)\,\|\,\pi_\theta(\cdot\mid ))$
using the K3 Monte-Carlo estimator from \citet{Schulman2020ApproximatingKL}. For
an action $a\sim\pi_{\mathrm{ref}}(\cdot\mid s)$ and ratio
$r(a,s)=\pi_\theta(a\mid s)/\pi_{\mathrm{ref}}(a\mid s)$, we use
$\widehat{D}_{\mathrm{KL}}^{(k3)}(s;a)=(r(a,s)-1)-\log r(a,s)$ and plug the
result into the KL-baseline difference defining $w_t^{(A)}$.
\subsubsection{Hyperparameters.}
We train TBPO on 4 H100 GPUs with batch size 32 and max response length 2048, using
RMSProp with a cosine learning-rate schedule (warmup ratio 0.05), lr $5\times10^{-7}$,
for 1 epoch. Unless stated otherwise, we use $\beta=0.1$ and Bregman parameters
$\lambda=0$ and $s=4$.

\subsection{Time Complexity}
\label{app:time_complexity}
We report wall-clock runtimes under the same 4$\times$H100 setup and batch
size used in our experiments. DPO finishes in 1h52m, while TDPO, TBPO-Q, and
TBPO-A take 2h29m, 2h52m, and 2h38m, respectively. TIS-DPO takes 5h40m because
it requires training auxiliary positive/negative DPO models before the main
preference-learning stage. These results show that TBPO is more expensive than
sequence-level DPO, as expected for a token-level objective, but its overhead is
modest relative to other token-level baselines: TBPO-A is close to TDPO, TBPO-Q
is only moderately slower, and both are substantially more efficient than
TIS-DPO.

\begin{table}[h!]
\centering
\begin{tabular}{lcc}
\toprule
\textbf{Method} & \textbf{Runtime} & \textbf{Relative to DPO} \\
\midrule
DPO & 1h52m & 1.00$\times$ \\
TDPO & 2h29m & 1.33$\times$ \\
TIS-DPO & 5h40m & 3.04$\times$ \\
TBPO-Q & 2h52m & 1.54$\times$ \\
TBPO-A & 2h38m & 1.41$\times$ \\
\bottomrule
\end{tabular}
\caption{Wall-clock training runtime under the same 4$\times$H100 setup and batch size.}
\label{tab:runtime}
\end{table}


\end{document}